\documentclass[preprint,12pt,authoryear]{elsarticle}

\usepackage{times}  % DO NOT CHANGE THIS
\usepackage{helvet}  % DO NOT CHANGE THIS
\usepackage{courier}  % DO NOT CHANGE THIS
\usepackage[hyphens]{url}  % DO NOT CHANGE THIS
\usepackage{booktabs}  % 提供三线表命令
\usepackage{amsthm}
\usepackage{makecell}
\newtheorem{theorem}{Theorem}

\newtheorem{assumption}{Assumption}

\newtheorem{proposition}{Proposition}

\usepackage{algorithm}
\usepackage{algpseudocode}

\usepackage{amsmath}
\usepackage{amssymb}
\usepackage{multirow}

\usepackage{tikz}
\usetikzlibrary{arrows.meta,positioning,calc,fit,backgrounds,
                shapes.geometric,shapes.multipart,decorations.pathreplacing}
\tikzset{
  >={Latex[round]},
  line/.style={-Latex,thick},
  dline/.style={-Latex,thick,densely dashed},
  module/.style={draw,rounded corners,thick,align=left,inner sep=6pt,fill=white},
  io/.style={module,fill=gray!10},
  fusion/.style={module,fill=green!12},
  star/.style={module,fill=yellow!18},
  util/.style={module,fill=purple!10},
  update/.style={module,fill=blue!8},
  playersbox/.style={draw,rounded corners,thick,dashed,inner sep=8pt},
  small/.style={font=\scriptsize},
  title/.style={font=\bfseries}
}

\usepackage{newfloat}
\usepackage{listings}

\usepackage[table]{xcolor}  % 带table选项的xcolor宏包，支持单元格着色
\usepackage{tabularx} % 表格扩展支持

\usepackage{graphicx} % DO NOT CHANGE THIS
\usepackage{natbib}  % DO NOT CHANGE THIS AND DO NOT ADD ANY OPTIONS TO IT
\usepackage{caption} % DO NOT CHANGE THIS AND DO NOT ADD ANY OPTIONS TO IT
\usepackage{amssymb}
\usepackage{amsmath}
\journal{Expert Systems With Applications}

\begin{document}

\begin{frontmatter}



\title{Curiosity Meets Cooperation: A Game-Theoretic Approach to Long-Tail Multi-Label Learning} %% Article title


\author{Canran Xiao\textsuperscript{1,*},
        Chuangxin Zhao\textsuperscript{2},
        Zong Ke\textsuperscript{3},
        Fei Shen\textsuperscript{4}} %% Author name

\affiliation{\textsuperscript{1}% First affiliation
            organization={School of Cyber Science and Technology, Shenzhen Campus of Sun Yat-sen University},
            city={Shenzhen},
            postcode={518107}, 
            country={China}}

\affiliation{\textsuperscript{2}% Second affiliation  
            organization={Institute of Automation, Chinese Academy of Sciences},
            addressline={Institute of Automation}, 
            city={Beijing},
            postcode={100190}, 
            country={China}}

\affiliation{\textsuperscript{3}% Third affiliation
            organization={Department of Statistics and Data Science, National University of Singapore},
            addressline={National University of Singapore}, 
            postcode={119077}, 
            country={Singapore}}

\affiliation{\textsuperscript{4}% Fourth affiliation
            organization={NExT++ Research Centre, National University of Singapore}, 
            addressline={National University of Singapore},
            postcode={119077}, 
            country={Singapore}}

\cortext[1]{Corresponding author}
%\ead{xiaocanran999@gmail.com; zhaochuangxin2023@ia.ac.cn; a0129009@u.nus.edu; shenfei29@nus.edu.sg}

%% Abstract
\begin{abstract}
Long-tail imbalance is endemic to multi-label learning: a few head labels dominate the gradient signal, while the many rare labels that matter in practice are silently ignored.  We tackle this problem by casting the task as a cooperative potential game.  In our uriosity-Driven Game-Theoretic Multi-Label Learning (CD-GTMLL) framework, the label space is split among several cooperating ''players'' that share a global accuracy payoff yet earn additional curiosity rewards that rise with label rarity and inter-player disagreement.  These curiosity bonuses inject gradient on under-represented tags without hand-tuned class weights. We prove that gradient best-response updates ascend a differentiable potential and converge to tail-aware stationary points that tighten a lower bound on the expected \textit{Rare-F1}.  Extensive experiments on conventional benchmarks and three extreme-scale datasets show consistent state-of-the-art gains, delivering up to \(+4.3\%\) Rare-F1 and \(+1.6\%\) P@3 over the strongest baselines, while ablations reveal emergent division of labour and faster consensus on rare classes.  CD-GTMLL thus offers a principled, scalable route to long-tail robustness in multi-label prediction.
\end{abstract}

%% Keywords
\begin{keyword}
Multi-label classification \sep Long-tailed learning \sep Game-theoretic machine learning \sep Curiosity-driven learning \sep Cooperative ensemble fusion
\end{keyword}

\end{frontmatter}



\section{Introduction}
\label{sec:intro}
Multi-label classification (MLC)~\citep{zhang2013review,liu2021emerging,vergara2025multi} assigns multiple tags to each instance, powering applications from image recognition~\citep{zhao2021transformer,ma2024text,teng2025adaptive} and text categorization~\citep{maltoudoglou2022well,guo2021label} to functional genomics~\citep{du2022deep}. 
The per-label distribution is typically long-tailed~\citep{tarekegn2021review,de2024survey}: head labels dominate while tail labels appear sporadically. This imbalance is exacerbated in MLC because (i) co-occurring labels make resampling risky, and (ii) metrics like mAP favor head labels. 
As a result, standard optimizers~\citep{ridnik2021asymmetric} often learn head-biased boundaries, achieving high scores while failing on tail labels-problematic for safety-critical applications.

In practice the per-label sample counts follow a heavy-tailed distribution: a handful of head labels dominate the data, whereas the vast majority of tail labels appear only sporadically, as shown in Fig. \ref{fig:teaser}.  
This long-tail imbalance~\citep{tarekegn2021review,de2024survey} is particularly severe in the multi-label regime because (i) multiple labels co-occur within a single instance, so na\"ive resampling can destroy cross-label correlations, and (ii) evaluation metrics such as mAP or micro-F1 are disproportionately influenced by head labels, starving tail classes of gradient signal.  
Consequently, conventional optimizers ~\citep{ridnik2021asymmetric} that target average loss or accuracy often learn a head-biased decision boundary, yielding high headline scores while silently failing on the tail-an outcome that is unacceptable in safety-critical or comprehensive retrieval scenarios\citep{barandas2024evaluation}.

\begin{figure}[tb]
	\centering
	\includegraphics[width=0.8\linewidth]{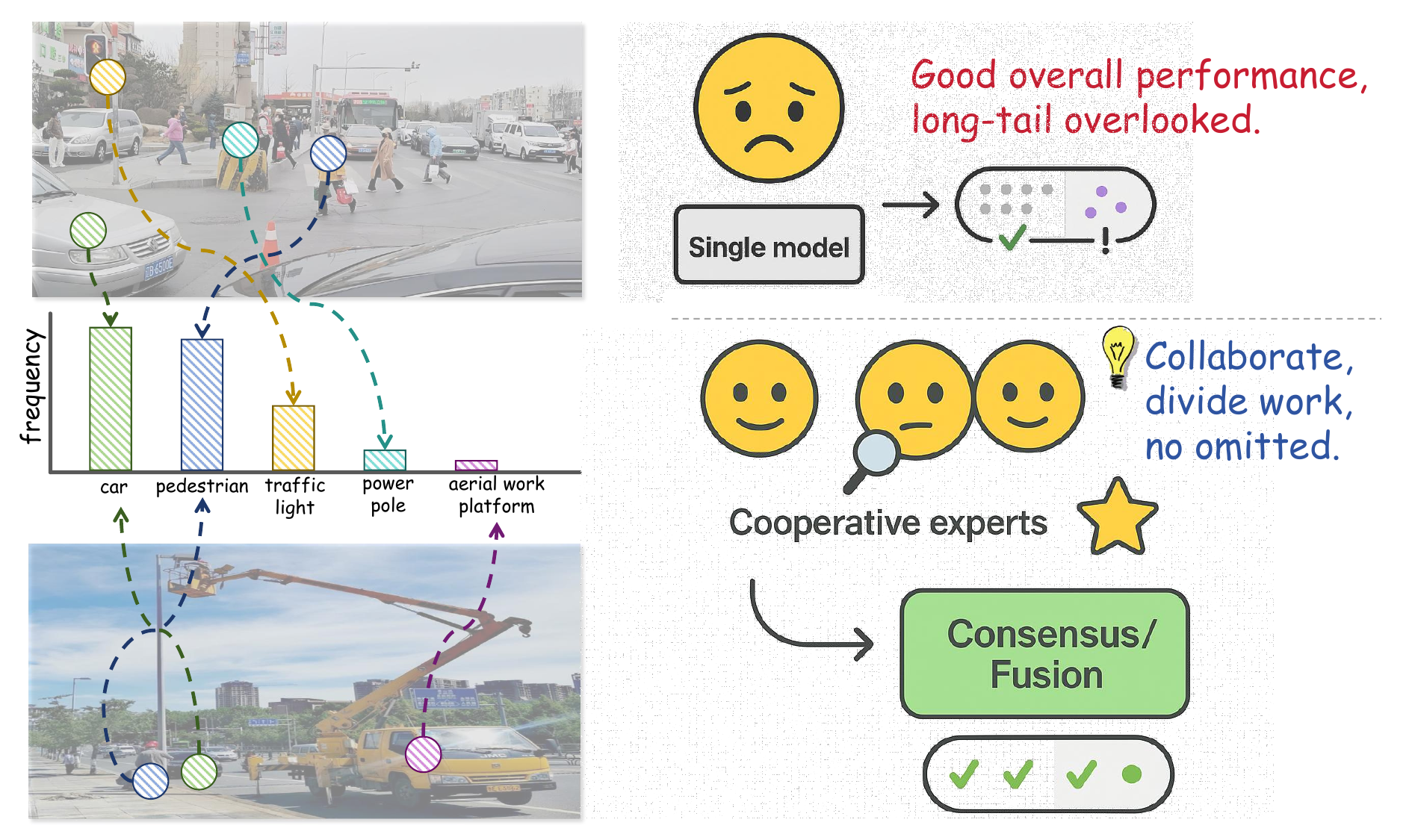}
	\caption{Real-world multi-label scenes are \emph{long-tailed}: a few frequent classes (e.g., \textit{car}, \textit{pedestrian}, \textit{traffic light}) dominate while many rare classes (e.g., \textit{power pole}, \textit{aerial work platform}) appear sparsely (left). A conventional single model may score well overall yet often overlooks these rare labels (top-right). Our approach replaces the monolith with a small set of cooperative experts that divide the work and then fuse their opinions into a single consensus prediction (bottom-right), improving tail coverage without sacrificing head performance.}
	\label{fig:teaser}
\end{figure}

Canonical imbalance remedies include  
(i) \emph{data rebalancing} via over/under-sampling or mix-up~\citep{liuclass,charte2015addressing,huang2021balancing,pan2024large},  
(ii) \emph{cost-sensitive objectives} such as class-balanced focal losses~\citep{zhao2023imbalanced,chen2022mcfl,zhang2024learning}, and  
(iii) \emph{label-space partition or ensemble} strategies~\citep{xu2019label,zou2024online}.  
However, these techniques falter when labels overlap: simplistic sampling distorts co-occurrence statistics~\citep{gao2024complementary,duarte2021plm}, re-weighting heuristics require brittle hyper-tunings, and global loss surrogates remain head-dominated.  
What is missing is a mechanism that \emph{persistently incentivizes exploration of tail labels while preserving the relational structure among labels}.

We view long-tail MLC as a \emph{cooperative multi-player game}.  
The label space is decomposed into overlapping subsets handled by distinct \emph{players} (sub-predictors).  
Each player maximizes a shared global payoff (overall multi-label accuracy) \emph{plus} an intrinsic \emph{curiosity reward} that grows with label rarity and inter-player disagreement.  
Framing the learning dynamics as a differentiable \emph{potential game}~\citep{fudenberg1991game} allows gradient-based best-response updates to converge to equilibria that jointly optimize head performance and tail discovery without manual re-weighting. The main contributions of this work are:

1)  We propose \textbf{Curiosity-Driven Game-Theoretic Multi-Label Learning} (\textbf{CD-GTMLL})-the first framework to cast long-tail MLC as a cooperative game equipped with a curiosity mechanism that systematically amplifies tail-label gradients. 

2) We provide theoretical guarantees that best-response updates ascend a global potential function encompassing both accuracy and curiosity, ensuring no player can ignore tail labels at equilibrium.

3) Extensive experiments on image and tabular benchmarks and their harder long-tail variants demonstrate consistent gains over state-of-the-art baselines.
Through ablations and game-behavior diagnostics, we elucidate how label-space decomposition induces specialized tail experts, accelerates rare-label correction, and yields interpretable multi-agent cooperation.

\section{Related Works}
\label{sec:related}

\subsection{Multi-label learning foundations}
A classical taxonomy distinguishes problem-transformation from algorithm-adaptation methods~\citep{tsoumakas2007overview,tsoumakas2010mining}. Transformation approaches reduce MLC to single-label problems: Binary Relevance (BR) trains one independent binary classifier per label, while Label Powerset (LP) treats each observed labelset as an atomic class; ensembles such as RAkEL randomly partition labels into $k$-sized subsets to balance tractability and dependency capture~\citep{read2011classifierchains,tsoumakas2010rakel,yao2024swift}. To better encode inter-label correlations without incurring LP's exponential blow-up, Classifier Chains (CC) pass earlier label predictions as features to later ones, with Ensembles of CC (ECC) improving robustness via random chain orders~\citep{read2011classifierchains}. For large label spaces, HOMER hierarchically partitions the output space to address scalability and class-imbalance issues~\citep{tsoumakas2008homer}. These foundations remain widely used baselines and inspire many modern variants\citep{yao2023ndc}.

Beyond transformations, MLC also adapts standard learners to multi-label outputs. A canonical example is ML-KNN, which augments $k$-NN with Bayesian label inference and remains a competitive non-neural baseline~\citep{zhang2007mlknn}. Neural precursors for text further showed that simple feed-forward architectures with sigmoid outputs already scale competitively on large corpora~\citep{nam2014large,zhang2025enhancing}.

With deep learning, many works explicitly learn label dependencies in images. CNN-RNN frameworks model co-occurrence and sequential dependencies between labels~\citep{wang2016cnnrnn}. Graph-based approaches (e.g., ML-GCN) propagate information over a label graph built from semantics or co-occurrence; adaptive variants learn the label graph itself~\citep{chen2019mlgcn,li2020agcn}. Transformers have become strong general-purpose MLC heads: C-Tran conditions attention on labels, while Query2Label uses label embeddings as decoder queries to `probe' class-specific evidence~\citep{lanchantin2021ctran,liu2021q2l}. These models achieve state-of-the-art results on balanced benchmarks but-without additional mechanisms-still risk head-label dominance under long-tail skew.

\subsection{Long-tail MLC}
\label{subsec:longtail_mlc}
Long-tailed label frequencies and inter-label co-occurrence make multi-label learning particularly fragile: head labels dominate the gradient and typical micro-averaged objectives reward head accuracy disproportionately, leaving rare labels under-trained. Prior work tackles this from several-largely orthogonal-angles.

A first line of work rescales contributions of positives/negatives or of rare/frequent labels. 
Distribution-Balanced Loss (DBL) corrects the negative-positive imbalance at the instance level and uses co-occurrence-aware weighting for multi-label data~\citep{wu2020distribution}. 
Asymmetric losses (ASL) suppress easy negatives and emphasize positives, yielding a strong off-the-shelf baseline for imbalanced MLC~\citep{ridnik2021asymmetric}. 
Further surrogates pursue instance- or batch-level reweighting and sampling (e.g., BalanceMix~\citep{song2024toward}, MLBOTE~\citep{teng2024multi}) or import ideas from single-label long-tail recognition-such as Focal/CB losses and margin/logit adjustments-to reduce head dominance and calibrate posterior scores~\citep{lin2017focal,cui2019classbalanced,cao2019ldam,menon2020logitadjustment}. 
While effective, these methods are typically driven by static heuristics (class priors, difficulty proxies) and treat labels mostly independently, making it hard to preserve cross-label structure under severe skew.

Another family modifies architectures or training dynamics to mitigate head-tail interference. 
Head-tail decoupling (e.g., LTMCP) separates representation learning from classifier adaptation across frequency regimes~\citep{yuan2019long}, HTTN leverages meta-transfer to transfer knowledge from head to tail classes~\citep{xiao2021does}. 
Decoupled training and two-stage schedules-learning features on imbalanced data then re-balancing the classifier-are common in long-tailed recognition and have been adapted to multi-label settings~\citep{kang2020decoupling}. 
In dense prediction, gradient-gating losses (Equalization/Seesaw) down-weight overwhelming negatives from frequent classes~\citep{tan2021eqlv2,wang2021seesaw}. Similar ideas can be ported to MLC but remain instance-level and static. 
Overall, architecture/schedule designs help but still lack an explicit mechanism that persistently steers capacity toward rare labels over training.

Data-side approaches synthesize or transplant statistics across head/tail classes, or regularize representations toward well-separated prototypes. 
LSFA transplants label-specific feature statistics to bolster tails~\citep{xu2023label}, MLC-NC encourages equiangular tight frame (ETF) structure for class prototypes to stabilize tails~\citep{tao2025mlc}. 
Contrastive and prototype-driven formulations (e.g., PaCo) have also proved effective for long-tail recognition by coupling uniformity with class-aware separation and can be combined with multi-label heads~\citep{cui2022paco}. 
These strategies improve feature geometry yet still rely on fixed priors/schedules rather than adaptive exploration of the tail.

When $L$ is in the tens of thousands, scalability becomes central. 
Transformer-based rankers (e.g., XR-Transformer) and label hierarchies (e.g., MatchXML) push state-of-the-art efficiency/accuracy, while ETU proposes generalized training objectives for extreme settings~\citep{zhang2021fast,ye2024matchxml,schultheis2023generalized}. 
Earlier XMC lines include hierarchical trees (Parabel/Bonsai), 1-vs-rest linear models with sparsity (DiSMEC), and deep label-tree models such as AttentionXML/LightXML~\citep{prabhu2018parabel,khandagale2020bonsai,babbar2017dismec,you2019attentionxml,jiang2021lightxml}. 
Despite strong micro-precision, these systems still optimize global surrogates and usually require additional heuristics to improve tail recall.

Across the spectrum-losses, architectures, data/representations, and XMC-the dominant paradigm relies on static reweighting/schedules or global surrogates. 
They may lift tail metrics, but they neither capture inter-label cooperation nor provide a persistent, adaptive drive toward under-represented labels. 
In contrast, our formulation casts long-tail MLC as a cooperative game with a curiosity signal that (i) continually redistributes learning pressure toward rare labels and (ii) preserves cross-label structure through overlapping players and consensus fusion. 
This dynamic treatment complements prior static remedies and explains our consistent Rare-F1 gains alongside stable head performance.

\subsection{Game-theoretic machine learning}
\label{subsec:gtml}
Game-theoretic views shaped a broad swath of machine learning, from adversarial training and GANs to multi-agent reinforcement learning (MARL)~\citep{yang2020overview,goodfellow2020generative,albrecht2024multi}. 
Beyond zero-sum settings, recent work studies differentiable $n$-player games and their gradient dynamics, decomposing vector fields into symmetric/antisymmetric parts and analyzing convergence or cycling under simultaneous updates~\citep{balduzzi2018mechanics,mescheder2018convergence}. 
On the cooperative side, potential games provide a powerful template: whenever each agent's unilateral improvement increases a shared potential, simple best-response or coordinate-ascent procedures provably ascend that potential~\citep{monderer1996potential}. 
This perspective underlies several representation-learning and coordination mechanisms in ML, including recent cooperative formulations for learning shared embeddings~\citep{slumbers2023game}.

In MARL, intrinsic motivation and reward shaping are classical tools to encourage exploration or coordination when extrinsic rewards are sparse or biased~\citep{ng1999policy,devlin2012potential,pathak2017curiosity,foerster2018counterfactual}. 
However, most shaping signals are instance- or state-centric and do not directly target label-space sparsity in supervised multi-label classification.

We cast long-tail MLC itself as a \emph{differentiable cooperative potential game}: overlapping sub-predictors (`players') share a global utility while receiving a \emph{curiosity} signal that persistently increases attention to rare labels. 
Unlike re-weighting or augmentation (\S\ref{subsec:longtail_mlc}), our formulation provides an adaptive, game-theoretic mechanism for tail exploration together with a convergence-friendly potential, yielding rare-label gains without sacrificing head performance.

\begin{figure}[th]
	\centering
	\includegraphics[width=\linewidth]{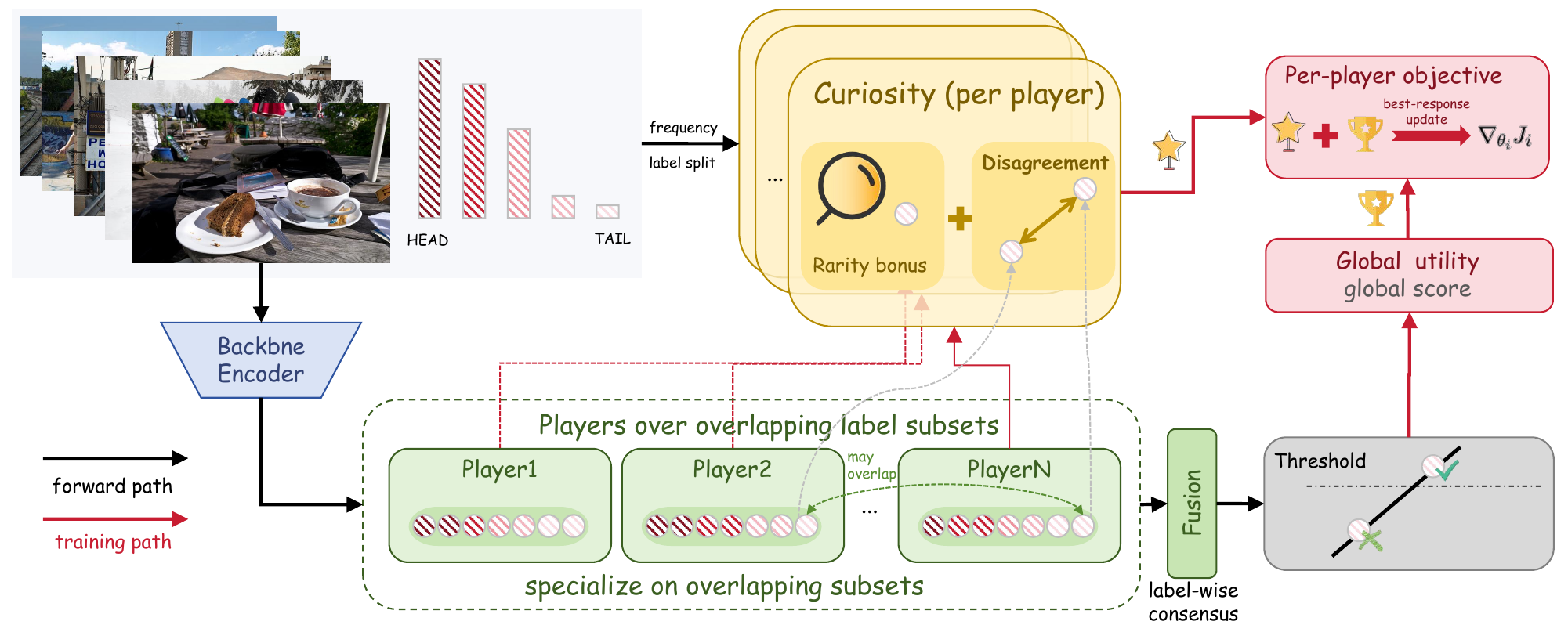}
	\caption{\textbf{CD--GTMLL pipeline.}  Long-tailed data are encoded by a shared Backbone into features $h(x)$.  In the forward path (black arrows), $h(x)$ is sent to multiple Players, each specializing on an overlapping label subset $\mathcal L_i$ and producing posteriors $\pi_i$.  A label-wise Fusion $g(\{\pi_i\})$ yields $\hat{\mathbf p}$, which is Thresholded to predictions $\hat{\mathbf y}$.  In the training path (red arrows), each player receives a Curiosity signal-rarity bonus (focus on tail labels) plus disagreement with peers-producing $C_i(x)$.  A Global utility $R$ scores $(\hat{\mathbf y},\mathbf y)$; the Per-player objective $J_i=R+\alpha\,\mathbb{E}[C_i]$ drives a cyclic best-response update of $\theta_i$.  The cooperative game encourages division of labor across overlapping subsets and improves tail coverage while preserving head performance.}
	\label{fig:pipeline}
\end{figure}

\section{Method}
\label{sec:method}
We formalize Curiosity-Driven Game-Theoretic Multi-Label Learning (CD-GTMLL), a cooperative framework that continuously steers learning toward the \emph{tail} of a long-tailed label distribution by coupling multiple predictors with a curiosity signal.  We first define the long-tail setting, then describe our $N$-player decomposition and global fusion scheme. The pipeline of CD-GTMLL is shown in Fig. \ref{fig:pipeline}.

\subsection{Long-Tail Multi-Label Formulation}

\paragraph{Notation}
Let $\mathcal{X}\!\subset\!\mathbb{R}^{d}$ denote the input space and $\mathcal{L}=\{1,\dots,L\}$ the label set.  
For an instance $\mathbf{x}\in\mathcal{X}$ the ground-truth annotation is a binary vector $\mathbf{y}\in\{0,1\}^{L}$, where $y_\ell=1$ indicates the presence of label~$\ell$.  
The data distribution over pairs $(\mathbf{x},\mathbf{y})$ is $\mathcal{D}$.

\paragraph{Head-tail split}
Empirically, label frequencies follow a power law.  Let
\begin{equation}
    \text{freq}(\ell)=\Pr_{(\mathbf{x},\mathbf{y})\sim\mathcal{D}}\bigl[y_\ell=1\bigr]
\end{equation}
be the marginal prevalence of label~$\ell$.  
We sort $\{\text{freq}(\ell)\}_{\ell=1}^{L}$ in descending order and define a threshold (e.g., top 10 \% cumulative frequency) to partition labels into the \emph{head} $\mathcal{L}_{\text{H}}$ and the \emph{tail} $\mathcal{L}_{\text{T}}=\mathcal{L}\setminus\mathcal{L}_{\text{H}}$.  
Throughout, tail labels refer to any $\ell\in\mathcal{L}_{\text{T}}$ and typically satisfy $\text{freq}(\ell)\!\ll\!\text{freq}(\ell')$ for $\ell'\in\mathcal{L}_{\text{H}}$.

\paragraph{$N$-player label decomposition}
Rather than employ a monolithic predictor, we assign $N$ cooperating players to overlapping label subsets
\begin{equation}
    \bigl\{\mathcal{L}_1,\dots,\mathcal{L}_N\bigr\},\qquad 
    \bigcup\nolimits_{i=1}^{N}\mathcal{L}_i=\mathcal{L},
\end{equation}
allowing $\mathcal{L}_i\cap\mathcal{L}_j\!\neq\!\varnothing$ to provide redundancy on difficult labels.  
Player~$i$ is parameterized by $\theta_i\!\in\!\Theta_i$ and outputs posterior probabilities
$
    \pi_i\bigl(\mathbf{x};\theta_i\bigr)\in[0,1]^{|\mathcal{L}_i|}.
$
For $\ell\!\in\!\mathcal{L}_i$, the component $[\pi_i(\mathbf{x};\theta_i)]_{\ell}$ estimates $\Pr(y_\ell=1\mid\mathbf{x})$ according to player~$i$.

\paragraph{Global fusion}
The ensemble prediction aggregates all players via a differentiable fusion operator
\begin{equation}
\label{eq:fusion}
    \hat{\mathbf{p}}
    \;=\;
    g\!\Bigl(\{\pi_i(\mathbf{x};\theta_i)\}_{i=1}^{N}\Bigr)\in[0,1]^{L},
\end{equation}
where a common choice is the weighted average
\begin{equation}
\hat{p}_\ell=\sum_{i:\,\ell\in\mathcal{L}_i}\omega_{i,\ell}\,[\pi_i(\mathbf{x};\theta_i)]_{\ell},\;
    \sum_{i:\,\ell\in\mathcal{L}_i}\omega_{i,\ell}=1.
\end{equation}
A binary decision is obtained by thresholding: $\hat{y}_\ell=\mathbf{1}\{\hat{p}_\ell>\tau\}$, with a constant or label-adaptive $\tau$.

\subsection{Game-Theoretic Formulation}
\label{sec:method:game}

\paragraph{Cooperative objective}
We model CD-GTMLL as an $N$-player cooperative game in which every player shares the same payoff  
\begin{equation}
\label{eq:global_payoff}
   R\!\bigl(\{\theta_i\}\bigr)
   \;=\;
   \mathbb{E}_{(\mathbf{x},\mathbf{y})\sim\mathcal{D}}
   \!\Bigl[
        \mathcal{M}\!\bigl(\hat{\mathbf{y}},\mathbf{y}\bigr)
   \Bigr],
\end{equation}
where $\mathcal{M}$ is a differentiable surrogate of a multi-label score (e.g., soft F1 or logistic loss).  Because tail labels occur rarely, their gradients are typically down-weighted when optimizing~\eqref{eq:global_payoff}; our curiosity term (introduced in \S\ref{sec:method:curiosity}) counteracts this bias.

\begin{assumption}[Continuity \& compactness~\citep{kawaguchi2016deep,zhang2020gradientdice}]
\label{assump:continuity}
(i) Each parameter set $\Theta_i$ is non-empty, compact, and convex.
(ii) $R(\{\theta_i\})$ is continuous on the product space $\prod_{i=1}^{N}\Theta_i$.
(iii) $\mathcal{M}$ is \emph{tail-responsive}: any increase in a tail label's accuracy yields a (possibly small) increase in $\mathcal{M}$.
\end{assumption}

% --- place this after Assumption~\ref{assump:continuity} ---
\begin{assumption}[Local improvability \& bounded predictions]
\label{assump:improvability}
(\textbf{i}) (\emph{Bounded predictions / clipping}) There exists $\varepsilon\in(0,\tfrac12)$ such that fused probabilities satisfy
$\hat p_\ell(\mathbf x)\in[\varepsilon,\,1-\varepsilon]$ almost surely.
(\textbf{ii}) (\emph{Local improvability on tail labels}) For every tail label $\ell\in\mathcal L_T$, any parameter profile $\boldsymbol\theta$, and any measurable set
$S\subset\{\mathbf x:\,y_\ell=1\}$ with $\Pr(S)>0$, there exists a player $k$ with $\ell\in\mathcal L_k$ and a feasible direction
$\mathbf v_k\in T_{\theta_k}\Theta_k$ such that the one-sided Gâteaux derivative of the fused probability along $\mathbf v_k$ obeys
$
\dot p_\ell(\mathbf x)\triangleq \left.\tfrac{\mathrm d}{\mathrm dt}\right|_{t=0^+}
\hat p_\ell(\mathbf x;\theta_1,\ldots,\theta_k+t\mathbf v_k,\ldots,\theta_N)\ge c_0\,\mathbf 1_S(\mathbf x)
$
for some $c_0>0$, and $\dot p_j(\mathbf x)\equiv 0$ for all $j\neq \ell$.
\end{assumption}

\begin{theorem}[Existence and tail-awareness]
\label{thm:existence_tail}
Under Assumptions~\ref{assump:continuity}--\ref{assump:improvability}, the following hold for the cooperative payoff in~\eqref{eq:global_payoff}:
\begin{enumerate}
\item (\emph{Existence}) $R$ admits a global maximizer $(\theta_1^\star,\ldots,\theta_N^\star)\in\prod_i\Theta_i$.
\item (\emph{Nash}) Every global maximizer is a pure-strategy Nash equilibrium.
\item (\emph{Tail-awareness}) For the rarity-weighted logistic utility $\mathcal M(\hat{\mathbf y},\mathbf y)=\tfrac{1}{Z}\sum_{\ell} w_\ell\big[y_\ell\log\hat p_\ell+(1-y_\ell)\log(1-\hat p_\ell)\big]$
with $w_\ell>0$, any global maximizer cannot systematically ignore tail labels:
for each $\ell\in\mathcal L_T$ with $\Pr(y_\ell=1)>0$ and any $\tau\in(\varepsilon,1-\varepsilon)$,
$\Pr\!\big(y_\ell=1\land \hat p_\ell^\star>\tau\big)>0$.
\end{enumerate}
\end{theorem}

\begin{proof}[Proof sketch]
(1) \emph{Existence:} $\Theta=\prod_i\Theta_i$ is compact and nonempty; $R$ is continuous by Assumption~\ref{assump:continuity}, hence attains a maximum (Weierstrass).  
(2) \emph{Nash:} identical-payoff game: any unilateral deviation of player $i$ cannot increase the global maximum value $R(\boldsymbol\theta^\star)$; thus $\boldsymbol\theta^\star$ is a pure NE.  
(3) \emph{Tail-awareness:} suppose by contradiction that at a maximizer there exists a tail label $\ell$ and threshold $\tau\in(\varepsilon,1-\varepsilon)$ such that $\hat p_\ell^\star\le\tau$ almost surely on $\{y_\ell=1\}$. By Assumption~\ref{assump:improvability}(ii) there is a feasible direction that increases $\hat p_\ell$ on a positive-measure subset $S\subset\{y_\ell=1\}$. Differentiating the rarity-weighted logistic objective along this direction (bounded by (i) and using dominated convergence) yields a strictly positive one-sided directional derivative:
\[
R'(0^+)=\mathbb E\!\Big[\tfrac{w_\ell}{Z}\big(\tfrac{y_\ell}{\hat p_\ell^\star}-\tfrac{1-y_\ell}{1-\hat p_\ell^\star}\big)\dot p_\ell\Big]
\;\ge\; \tfrac{w_\ell c_0}{Z\tau}\Pr(S) \;>\; 0,
\]
contradicting optimality. Hence no equilibrium can be tail-blind. For the complete proof, please refer to \ref{AP:A}. \qedhere
\end{proof}

\subsection{Curiosity-Driven Exploration for Tail Labels}
\label{sec:method:curiosity}

\paragraph{Curiosity reward}
To continuously steer learning toward the tail, each player receives a \emph{curiosity} bonus that (i) emphasizes correctness on infrequent labels and (ii) encourages useful diversity with peers.
\begin{equation}
\label{eq:curiosity_smooth}
\begin{split}
C_i(\mathbf x) = &\underbrace{\sum_{\ell\in\mathcal L_i}\frac{1}{1+\mathrm{freq}(\ell)}
\big[y_\ell\log p_{i\ell}(\mathbf x)+(1-y_\ell)\log\!\big(1-p_{i\ell}(\mathbf x)\big)\big]}_{\text{rarity{\text-}weighted log-likelihood}} \\
&\quad + \beta\,D\!\big(\pi_i(\mathbf x),\,\overline{\pi}_{-i}(\mathbf x)^{\text{stop}}\big),
\end{split}
\end{equation}
where $p_{i\ell}=[\pi_i]_\ell\in(0,1)$ is the player-$i$ posterior for label $\ell$ and 
$D(\cdot\|\cdot)$ is a divergence (e.g., KL) computed on overlapping labels; the superscript ``$\text{stop}$'' indicates a stop-gradient so that $C_i$ depends on $\theta_i$ only, preserving the potential-game property.  The coefficient $\beta\!\ge\!0$ trades off rarity emphasis and disagreement.

\paragraph{Per-player objective}
With curiosity, player $i$ maximizes
\begin{equation}
\label{eq:player_obj}
J_i(\theta_i)=R(\{\theta_j\})+\alpha\,\mathbb E_{\mathbf x\sim\mathcal D}\!\big[C_i(\mathbf x)\big],\qquad \alpha>0.
\end{equation}

\paragraph{Potential function}
To relate block-wise improvements to a single global scalar, we define the potential
\begin{equation}
\label{eq:potential}
   \Phi(\{\theta_i\})
   \;=\;
   R\bigl(\{\theta_i\}\bigr)
   \;+\;
   \alpha \sum_{i=1}^{N}
   \mathbb{E}_{\mathbf{x}\sim\mathcal{D}}\!\bigl[C_i(\mathbf{x})\bigr].
\end{equation}
Because $C_i$ is computed with a stop-gradient peer average (\S\ref{sec:method:curiosity}), it depends only on $\theta_i$, hence
$\nabla_{\theta_i}\Phi \equiv \nabla_{\theta_i}J_i$ with $J_i$ in~\eqref{eq:player_obj}.

\begin{proposition}[Curiosity prioritizes tail labels]
\label{prop:rare_priority}
Assume Assumption~\ref{assump:improvability}(i) (bounded predictions) and let $\tau\in(\varepsilon,1-\varepsilon)$ be the decision threshold.  
Fix any tail label $\ell\in\mathcal L_T$ and let 
$S=\{\mathbf x: y_\ell=1,\ \hat p_\ell(\mathbf x)\le\tau\}$ be the set of \emph{tail false negatives}.  
If $\Pr(S)>0$, then there exists a player $k$ with $\ell\in\mathcal L_k$ such that the partial derivative of $J_k$ in the logit direction $z_{k\ell}$ (with $p_{k\ell}=\sigma(z_{k\ell})$) satisfies
\[
\frac{\partial J_k}{\partial z_{k\ell}}
\ \ge\ 
\alpha\,\frac{1}{1+\mathrm{freq}(\ell)}\,\varepsilon\,\Pr(S)
\;>\;0.
\]
Consequently, cyclic best responses cannot stagnate while a tail label has a nonzero mass of false negatives.
\end{proposition}

\begin{proof}[Proof sketch]
On $S$ we have $y_\ell=1$ and, by bounded predictions, $p_{k\ell}\le 1-\varepsilon$.  
For the rarity term in~\eqref{eq:curiosity_smooth}, 
$\partial\,[y\log p+(1-y)\log(1-p)]/\partial z = y-p$; hence on $S$,
$\partial C_k/\partial z_{k\ell}\ge \tfrac{1}{1+\mathrm{freq}(\ell)}(1-p_{k\ell})\ge \tfrac{\varepsilon}{1+\mathrm{freq}(\ell)}$.  
Taking expectation over $S$ and multiplying by $\alpha$ gives the stated lower bound; 
the $D$-term is not needed for positivity. Refer to \ref{AP.B} for the complete proof.\qedhere
\end{proof}

\section{Learning Algorithm}
In this section, we describe how to solve the cooperative game defined in \S\ref{sec:method} via a best-response-style iterative procedure.

\subsection{Algorithmic Framework}
\label{sec:method:algo}

\paragraph{Best-response update}
Fixing the parameters of all other players, player~$i$ solves
\begin{equation}
\label{eq:best_resp}
   \theta_i^{\star}
   =
   \mathop{\mathrm{arg\,max}}_{\theta_i \in \Theta_i}
   \Bigl\{
       R\bigl(\{\theta_j\}_{j \neq i}, \theta_i\bigr)
       + \alpha \, \mathbb{E}_{\mathbf{x} \sim \mathcal{D}} \bigl[C_i(\mathbf{x})\bigr]
   \Bigr\},
\end{equation}
which we approximate by \emph{one} gradient-ascent (or a short inner loop):
\begin{equation}
\label{eq:grad_step}
   \theta_i
   \gets
   \theta_i
   +
   \eta_i
   \nabla_{\theta_i}
   \Bigl[
       R\bigl(\{\theta_j\}, \theta_i\bigr)
       + \alpha \,\mathbb{E}_{\mathbf{x}\sim\mathcal D}[C_i(\mathbf x)]
   \Bigr],
\end{equation}
where the gradient is taken with peers \(\{\theta_{j\ne i}\}\) \emph{detached} inside $C_i$ and the fusion $g$ (cf.\ \eqref{eq:curiosity_smooth}).

\paragraph{Cyclic best-response}
Players are updated sequentially \(i=1\!\to\!2\!\to\!\cdots\!\to\!N\!\to\!1\!\cdots\).
This is a cyclic coordinate ascent on the joint potential (\S\ref{sec:method:conv}). In practice we use Adam and \texttt{inner\_iters}=1, which is stable and efficient.

\begin{algorithm}[t]
\caption{\textsc{CD--GTMLL Training}}
\label{alg:cdgtmll}
\begin{algorithmic}[1]
\Require Dataset $\mathcal D$; learning rates $\{\eta_i\}_{i=1}^N$; curiosity hyperparams $(\alpha,\beta)$; fusion weights $\omega$
\Ensure $\theta_0,\{\theta_i\}_{i=1}^N,\omega$
\State Precompute $\mathrm{freq}(\ell)$ and head/tail split; init backbone $\theta_0$, heads $\{\theta_i\}$, fusion $\omega$
\While{not converged}
  \State Sample minibatch $\mathcal B=\{(\mathbf x^m,\mathbf y^m)\}_{m=1}^B$
  \State \textbf{Shared forward:} compute features $\mathbf h(\mathbf x;\theta_0)$ for $\mathbf x\in\mathcal B$
  \State \textbf{Players:} for $j=1,\ldots,N$, get posteriors $p_j(\mathbf x)=\pi_j(\mathbf x;\theta_j)$
  \State Cache detached copies $p_j^{\perp}$ \Statex \(\triangleright\) peers are treated as \emph{stop-grad} inside curiosity
  \State \textbf{Fuse:} $\hat{\mathbf p}\leftarrow g(\{p_j\},\omega)$ \emph{via} \eqref{eq:fusion}
  \State \textbf{Batch utility:} $R_{\mathcal B}\leftarrow \mathcal M(\hat{\mathbf p},\mathbf y)$ \emph{via} \eqref{eq:global_payoff}
  \For{$i=1,\ldots,N$} \Statex \(\triangleright\) cyclic best response on player $i$
    \State Build peer avg.\ $\overline p_{-i}^{\perp}$ from $\{p_j^{\perp}\}_{j\neq i}$ (overlapping labels)
    \State \textbf{Curiosity:} $C_i(\mathbf x)$ \emph{via} smooth form \eqref{eq:curiosity_smooth} using $p_i$ and $\overline p_{-i}^{\perp}$
    \State \textbf{Player obj.:} $J_i \leftarrow R_{\mathcal B}+\alpha\,\mathbb E_{\mathbf x\in\mathcal B}[C_i(\mathbf x)]$ \emph{via} \eqref{eq:player_obj}
    \State \textbf{Update head:} $\theta_i \leftarrow \theta_i + \eta_i\,\nabla_{\theta_i} J_i$ \Statex \(\triangleright\) treat $\theta_{j\neq i},\omega,\theta_0$ as constants here
  \EndFor
  \State \textbf{Backbone/fusion step:} ascend $\nabla_{\theta_0,\omega}\Phi$ using \eqref{eq:potential}
\EndWhile
\State \Return $\theta_0,\{\theta_i\}_{i=1}^N,\omega$
\end{algorithmic}
\end{algorithm}

\paragraph{Computational complexity}
Let $C_{\text{bb}}$ be the cost of one backbone forward per sample and $C_{\text{head}}$ the cost of one head forward/backward for a \emph{single} binary label.\footnote{When heads are linear classifiers on fixed features, $C_{\text{head}}$ is tiny compared to $C_{\text{bb}}$.}
One outer iteration with cached features costs
\begin{equation}
\mathcal O\!\Big(B\,C_{\text{bb}}\Big)\;+\;
\mathcal O\!\Big(C_{\text{head}}\sum_{i=1}^N|\mathcal L_i|\Big)\;+\;
\mathcal O\!\Big(\sum_{i=1}^N|\mathcal O_i|\Big),
\end{equation}
where $\mathcal O_i=\{\ell\in\mathcal L_i:\exists j\neq i,\ \ell\in\mathcal L_j\}$ denotes overlapping labels (for the $D$ term and fusion).  
With mild overlap $\sum_i|\mathcal L_i|=\Theta(mL)$ and $\sum_i|\mathcal O_i|=\Theta(\rho L)$ for an average coverage factor $m\approx 1+\rho$ ($\rho\ll 1$).  
Thus each sweep remains \emph{linear} in $L$ (plus one backbone forward), matching the order of a monolithic model while improving tail coverage.  

\paragraph{Inference procedure}
At test time, each player produces $p_i(\mathbf x_{\mathrm{te}})$, we fuse to $\hat{\mathbf p}\in[0,1]^L$ by \eqref{eq:fusion} and threshold $\hat y_\ell=\mathbf 1\{\hat p_\ell>\tau_\ell\}$.  
Because curiosity already induces tail specialization during training, no extra calibration is needed. Inference is one forward per player plus light fusion.

\subsection{Convergence Analysis}
\label{sec:method:conv}

Recall the potential \eqref{eq:potential}, 
because $C_i$ uses a stop-gradient peer average (\S\ref{sec:method:curiosity}), it depends only on $\theta_i$. Hence
\(
\nabla_{\theta_i}\Phi \equiv \nabla_{\theta_i}J_i
\)
with $J_i$ from \eqref{eq:player_obj}. A best-response (block) step that increases $J_i$ thus never decreases $\Phi$.

\begin{assumption}[Block smoothness \& boundedness]
\label{assump:reg}
For each player $i$:
(i) $\Phi$ is continuously differentiable on the compact product set $\Theta=\prod_{j=1}^N\Theta_j$; 
(ii) the block gradient is $L_i$-Lipschitz, i.e.,
$\|\nabla_{\theta_i}\Phi(\theta_1,\!\ldots,\theta_i+\Delta,\!\ldots,\theta_N)-\nabla_{\theta_i}\Phi(\boldsymbol\theta)\|\le L_i\|\Delta\|$;
(iii) the block stepsizes satisfy $0<\eta_i\le \tfrac{1}{L_i}$ (or Armijo backtracking is used);
(iv) $\Phi$ is bounded above on $\Theta$.
\end{assumption}

\begin{theorem}[Convergence of cyclic block ascent]
\label{thm:conv}
Under Assumption~\ref{assump:reg}, the iterates produced by Alg.~\ref{alg:cdgtmll} (one gradient-ascent step per block in cyclic order) generate a monotone non-decreasing sequence $\{\Phi^{(t)}\}_{t\ge 0}$ that is bounded above. Consequently $\Phi^{(t)}$ converges, and every limit point $\boldsymbol\theta^\infty$ is a first-order stationary point of $\Phi$, i.e., $\nabla_{\theta_i}\Phi(\boldsymbol\theta^\infty)=\mathbf 0$ for all $i$.
\end{theorem}

\begin{proof}[Proof sketch]
Let $\boldsymbol\theta^{(t)}$ be the iterate before updating block $i$. By block-$L_i$ smoothness (Descent/Ascent Lemma) and the update $\Delta=\eta_i\nabla_{\theta_i}\Phi(\boldsymbol\theta^{(t)})$,
\begin{equation}
\Phi\big(\boldsymbol\theta^{(t)}+\mathbf e_i\Delta\big)
\;\ge\;
\Phi\big(\boldsymbol\theta^{(t)}\big)
\;+\;\eta_i\!\left(1-\tfrac{L_i\eta_i}{2}\right)\!\big\|\nabla_{\theta_i}\Phi(\boldsymbol\theta^{(t)})\big\|^2,
\end{equation}
which is $\ge \Phi(\boldsymbol\theta^{(t)})$ since $\eta_i\!\le\!1/L_i$. Summing the guaranteed gains over one sweep yields a non-decreasing, bounded-above sequence $\{\Phi^{(t)}\}$; hence $\Phi^{(t)}$ converges and $\sum_t\|\nabla_{\theta_{i_t}}\Phi(\boldsymbol\theta^{(t)})\|^2<\infty$. Because each block is visited infinitely often, the block gradients must vanish along a subsequence, and by continuity they vanish at every limit point, proving stationarity. A detailed proof is provided in \ref{AP.C}. \qedhere
\end{proof}

%===================== Main text (short + sketch) =====================
\subsection{Relationship Between the Objective and \textit{Rare\,-\,F1}}
\label{sec:method:metric}

While the potential~\eqref{eq:potential} is optimized indirectly, evaluation ultimately relies on tail-sensitive F1.  
We formalize the link for the \emph{micro-averaged} tail F1 at threshold $\tau\in(0,1)$:
\begin{equation}
\widetilde F_{\mathrm T}(\tau)\;\triangleq\;\frac{2\,\mu_{\mathrm{TP},\mathrm T}(\tau)}{2\,\mu_{\mathrm{TP},\mathrm T}(\tau)+\mu_{\mathrm{FP},\mathrm T}(\tau)+\mu_{\mathrm{FN},\mathrm T}(\tau)},
\end{equation}
where $\mu_{\mathrm{TP},\mathrm T},\mu_{\mathrm{FP},\mathrm T},\mu_{\mathrm{FN},\mathrm T}$ are the micro counts aggregated across tail labels and then averaged over the data distribution.
Let $\mu_{\mathrm{Pos},\mathrm T}=\sum_{\ell\in\mathcal L_{\mathrm T}}\Pr(y_\ell=1)$ be the tail positive mass, $w_{\min,\mathrm T}=\min_{\ell\in\mathcal L_{\mathrm T}} w_\ell$, and $Z=\sum_{\ell=1}^L w_\ell$ the normalizer in~\eqref{eq:global_payoff}.  
We adopt $\varepsilon$-clipping as in Assumption~\ref{assump:improvability}(i) so that $\hat p_\ell\in[\varepsilon,1-\varepsilon]$.

\begin{theorem}[Lower bound on micro Rare--F1]
\label{thm:microF1_bound}
Define $\kappa(\tau)=\max\!\big\{1/[-\log(1-\tau)],\,1/[-\log\tau]\big\}$.  
Then, for the rarity-weighted logistic utility in~\eqref{eq:global_payoff}, the micro tail F1 satisfies
\begin{equation}
\label{eq:microF1_main_bound}
\widetilde F_{\mathrm T}(\tau)\;\ge\;
1\;-\;\frac{\kappa(\tau)\,Z}{2\,\mu_{\mathrm{Pos},\mathrm T}\,w_{\min,\mathrm T}}\;\bigl(-R(\{\theta_i\})\bigr).
\end{equation}
\end{theorem}

\paragraph{Proof sketch}
(i) For any $\tau\!\in\!(\varepsilon,1-\varepsilon)$, the threshold events are controlled by the logistic tails:
$\mathbf 1\{\hat p_\ell\!\ge\!\tau,y_\ell\!=\!0\}\le \frac{-\log(1-\hat p_\ell)}{-\log(1-\tau)}\mathbf 1\{y_\ell\!=\!0\}$ and
$\mathbf 1\{\hat p_\ell\!<\!\tau,y_\ell\!=\!1\}\le \frac{-\log \hat p_\ell}{-\log\tau}\mathbf 1\{y_\ell\!=\!1\}$, taking expectations gives micro $\mathrm{FP}$/$\mathrm{FN}$ bounds.  
(ii) The sum of tail-conditioned logistic losses is upper-bounded (up to a constant) by $-R$ via the weights $w_\ell$ in~\eqref{eq:global_payoff}.  
(iii) Using $\widetilde F_{\mathrm T}=1-\frac{\mu_{\mathrm{FP},\mathrm T}+\mu_{\mathrm{FN},\mathrm T}}{2\mu_{\mathrm{Pos},\mathrm T}+\mu_{\mathrm{FP},\mathrm T}-\mu_{\mathrm{FN},\mathrm T}}\ge 1-\frac{\mu_{\mathrm{FP},\mathrm T}+\mu_{\mathrm{FN},\mathrm T}}{2\mu_{\mathrm{Pos},\mathrm T}}$ and collecting the bounds yields~\eqref{eq:microF1_main_bound}.  
A complete proof with all inequalities is provided in Appendix~D.
\medskip

\paragraph{Implication}
Any ascent in $R$ (hence in the potential $\Phi$ by \eqref{eq:potential}) tightens the lower bound~\eqref{eq:microF1_main_bound}, providing a direct guarantee that optimizing the cooperative objective improves micro Rare--F1.  Choosing $\tau$ to minimize $\kappa(\tau)$ (e.g., $\tau\!=\!1/2$ under clipping) strengthens the constant; using $R_{\mathrm T}$ tightens it further.

\section{Experiments}
\subsection{Experimental Setup}
\label{subsec:exp_setup}
\paragraph{Datasets}
\textbf{(i) Conventional MLC} Pascal VOC 2007 (\(9\,{,}963\) images, 20 labels), MS-COCO 2014 (\(82\,{,}081\) / \(40\,{,}504\) train/val images, 80 labels), Yeast (\(2\,{,}417\) instances, 14 labels) and Mediamill (\(43\,{,}907\) videos, 101 labels).  \textbf{(ii) Rare-focused MLC.} To stress tail robustness we create Pascal VOC-R, COCO-R, Yeast-R and Mediamill-R by down-sampling positives of the least-frequent labels (details in Appendix.E.3).
\textbf{(iii) Extreme MLC.} Eurlex-4K (55k EU legal documents, 3\,984 labels), Wiki10-31K (20k Wikipedia articles, 30\,938 labels) and AmazonCat-13K (1.2M Amazon product descriptions, 13\,330 labels).  Image datasets use standard random crop-flip augmentation; tabular/text data are tokenised and $\ell_2$-normalised. 

\paragraph{Evaluation metrics}
For conventional MLC we report mean Average Precision (mAP), Micro-/Macro-F1 and \textit{Rare-F1} computed on the bottom 20\% of labels.  For XMC we follow the literature and use Precision@k (\(k\!\in\!\{1,3,5\}\)). The detailed definitions of these evaluation metrics are provided in Appendix.E.4.

\paragraph{Baselines}
(1) \textit{Binary-cross-entropy:} BCE.  
(2) \textit{Direct F-measure optimisers:} Plugin-Estimator (PE) \citep{koyejo2015consistent}, Surrogate RB (SRB) \citep{kotlowski2016surrogate}.  
(3) \textit{Long-tail specific:} DBL \citep{wu2020distribution}, LSFA \citep{xu2023label}, MLC-NC \citep{tao2025mlc}.  
(4) \textit{Imbalance losses:} ASL \citep{ridnik2021asymmetric}, BalanceMix \citep{song2024toward}, MLBOTE \citep{teng2024multi}.  
(5) \textit{Extreme MLC:} XR-Transformer \citep{zhang2021fast}, MatchXML \citep{ye2024matchxml}, ETU \citep{schultheis2023generalized}.  
(6) \textit{Generic multi-label models:} CC, ML-GCN, C-Tran, ML-Decoder, LCIFS.

Refer to Appendix.E for more implementation details.

\begin{table*}[t]
\centering
\resizebox{\textwidth}{!}{ % 添加这一行来调整表格宽度
\begin{tabular}{lcccccccc}
\toprule
\multirow{2}{*}{Method} &
\multicolumn{2}{c}{Pascal VOC} &
\multicolumn{2}{c}{COCO} &
\multicolumn{2}{c}{Yeast} &
\multicolumn{2}{c}{Mediamill}\\
\cmidrule(lr){2-3}\cmidrule(lr){4-5}\cmidrule(lr){6-7}\cmidrule(lr){8-9}
 & mAP & Rare-F1 & mAP & Rare-F1 & Micro-F1 & Rare-F1 & Macro-F1 & Rare-F1\\
\midrule
BCE                 & 88.9 & 71.0 & 63.0 & 41.2 & 75.4 & 64.8 & 49.3 & 35.4\\
PE~\citep{koyejo2015consistent}           & 90.5 & 75.0 & 65.1 & 44.9 & 77.1 & 66.0 & 51.7 & 37.6\\
SRB~\citep{kotlowski2016surrogate}        & 90.7 & 75.4 & 65.5 & 45.3 & 77.5 & 66.4 & 52.0 & 37.9\\
DBL~\citep{wu2020distribution}            & 91.0 & 76.0 & 66.0 & 45.9 & 77.9 & 67.0 & 52.8 & 38.6\\
LSFA~\citep{xu2023label}                  & 91.3 & 77.0 & 66.2 & 46.2 & 78.2 & 67.3 & 53.1 & 39.0\\
MLC-NC~\citep{tao2025mlc}                 & \underline{91.9} & \underline{78.4} & 66.9 & 47.0 & \underline{79.2} & 69.0 & 55.2 & 41.5\\
ASL~\citep{ridnik2021asymmetric}          & 91.6 & 77.4 & 66.5 & 46.4 & 78.6 & 68.0 & 54.5 & 40.3\\
BalanceMix~\citep{song2024toward}         & 91.5 & 77.1 & 66.4 & 46.1 & 78.4 & 67.8 & 54.4 & 40.1\\
MLBOTE~\citep{teng2024multi}              & 89.6 & 72.9 & 63.4 & 42.3 & 76.5 & 66.1 & 49.7 & 36.2\\
CC~\citep{read2011classifier}             & 89.7 & 73.5 & 63.8 & 42.6 & 76.7 & 66.3 & 50.6 & 37.0\\
ML-GCN~\citep{chen2019multi2}             & 90.9 & 75.9 & 66.0 & 45.7 & 78.0 & 67.5 & 53.2 & 39.2\\
C-Tran~\citep{lanchantin2021general2}     & 91.2 & 76.3 & 66.2 & 46.0 & 78.1 & 67.6 & 53.9 & 39.8\\
ML-Decoder~\citep{ridnik2023ml}           & 91.8 & 77.5 & \underline{66.8} & 46.7 & 78.9 & 68.2 & 55.0 & 40.8\\
LCIFS~\citep{fan2024learning}             & 91.3 & 77.7 & 67.0 & \underline{47.2} & 78.4 & 67.5 & \underline{55.3} & \underline{40.9}\\
\textbf{CD--GTMLL}                        & \textbf{92.7} & \textbf{79.2} &
\textbf{68.4} & \textbf{49.2} &
\textbf{80.3} & \textbf{70.2} &
\textbf{56.1} & \textbf{42.8}\\
\bottomrule
\end{tabular}
} % 添加这一行闭合 resizebox
\caption{\textbf{Standard--frequency results} (\%). Rare-F1 is computed on the bottom 20 \% labels.  Best and second best are \textbf{bold} and \underline{underlined}.}
\label{tab:main_standard}
\end{table*}

\begin{figure}[htbp] % 这里需要位置参数，如htbp
    \centering
    \includegraphics[width=0.65\linewidth]{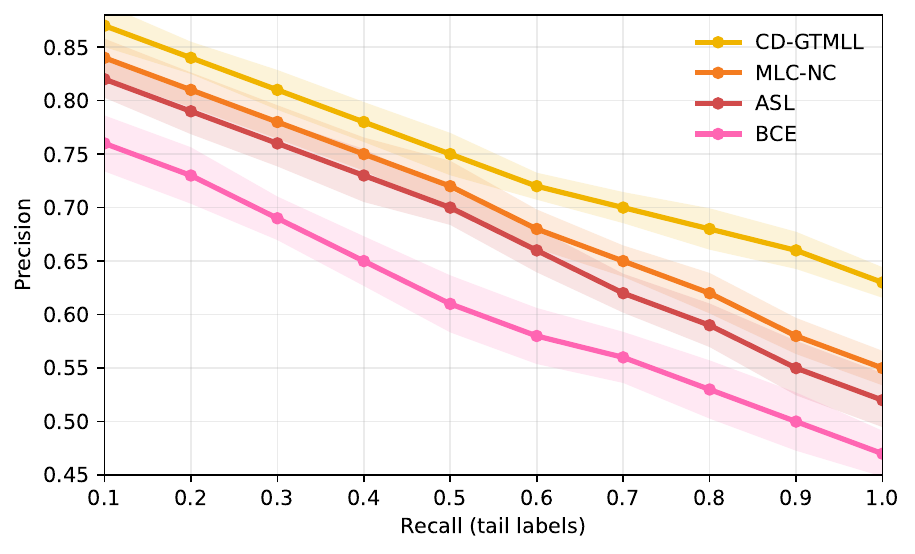} % 
    \caption{Precision-recall curves on \emph{Pascal VOC-R}.}
    \label{fig:pr_tail}
\end{figure}

\begin{table*}[t]
\centering
\scriptsize
\resizebox{\textwidth}{!}{ % 添加这一行来调整表格宽度
\begin{tabular}{lcccccccc}
\toprule
\multirow{2}{*}{Method} &
\multicolumn{2}{c}{Pascal VOC-R} &
\multicolumn{2}{c}{COCO-R} &
\multicolumn{2}{c}{Yeast-R} &
\multicolumn{2}{c}{Mediamill-R}\\
\cmidrule(lr){2-3}\cmidrule(lr){4-5}\cmidrule(lr){6-7}\cmidrule(lr){8-9}
& @30\% & @50\% & @30\% & @40\% & @40\% & @50\% & @40\% & @50\%\\
\midrule
BCE                         & 64.8 & 57.3 & 38.7 & 35.2 & 61.9 & 58.0 & 30.9 & 26.7\\
PE~\citep{koyejo2015consistent}   & 66.5 & 58.7 & 40.1 & 36.6 & 63.0 & 59.0 & 32.0 & 27.4\\
SRB~\citep{kotlowski2016surrogate}& 67.0 & 59.2 & 40.6 & 37.1 & 63.4 & 59.4 & 32.5 & 27.9\\
DBL~\citep{wu2020distribution}    & 68.0 & 60.2 & 41.5 & 37.9 & 64.1 & 60.1 & 33.2 & 28.6\\
LSFA~\citep{xu2023label}          & 68.7 & 61.0 & 42.1 & 38.6 & 64.6 & 60.6 & 33.6 & 29.1\\
MLC-NC~\citep{tao2025mlc}         & \underline{70.1} & \underline{62.8} & \underline{44.0} & \underline{40.6} & \underline{66.3} & \underline{62.0} & \underline{37.5} & \underline{32.5}\\
XR-Transformer~\citep{zhang2021fast}& 66.8 & 59.8 & 41.6 & 37.6 & 64.0 & 59.9 & 32.9 & 28.3\\
MatchXML~\citep{ye2024matchxml}   & 67.2 & 60.1 & 41.9 & 37.8 & 64.2 & 60.1 & 33.1 & 28.5\\
ASL~\citep{ridnik2021asymmetric}  & 69.2 & 62.4 & 42.9 & 39.7 & 65.5 & 61.6 & 36.4 & 31.4\\
BalanceMix~\citep{song2024toward} & 69.0 & 62.1 & 42.7 & 39.4 & 65.3 & 61.5 & 36.2 & 31.2\\
MLBOTE~\citep{teng2024multi}      & 65.4 & 58.5 & 39.2 & 35.8 & 62.8 & 59.3 & 31.7 & 27.4\\
\textbf{CD--GTMLL}                & \textbf{71.5} & \textbf{64.2} & \textbf{45.1} & \textbf{42.0} & \textbf{67.3} & \textbf{63.3} & \textbf{38.4} & \textbf{33.9}\\
\bottomrule
\end{tabular}
} % 添加这一行闭合 resizebox
\caption{\textbf{Rare-F1 (\%) on artificially down-sampled ``R'' variants}.  ``@\(30\%\)'' and ``@\(50\%\)'' indicate the fraction of positives removed.  Best / second best are \textbf{bold} / \underline{underlined}.}
\label{tab:main_rare}
\end{table*}

\subsection{Comparison with Existing Methods}
\label{subsec:exp_sota}

\paragraph{Conventional setting (standard frequency)}

As shown in Table \ref{tab:main_standard}, CD--GTMLL obtains the best
 mAP and the strongest Rare-F1 on every dataset, confirming that curiosity-driven cooperation yields superior tail recall without compromising head performance.

\paragraph{Rare-focused setting}
Table \ref{tab:main_rare} reports Rare-F1 on the down-sampled splits with \(30\%\) and \(50\%\) positive-rate reduction.  CD--GTMLL enlarges the tail margin at every severity level, beating the strongest prior (MLC-NC) by \(+1.4\text{-}1.7\%\) absolute at the hardest \(50\%\) setting.
The PR curves in Fig.~\ref{fig:pr_tail} complement the Rare-F1 numbers of Table~\ref{tab:main_rare}: CD-GTMLL traces an unambiguously higher frontier, retaining \(\approx\!0.70\) precision at 0.7 recall (\(+5\,\text{pp}\) over MLC-NC, \(+8\,\text{pp}\) over ASL, \(+14\,\text{pp}\) over BCE) and never being dominated at any operating point.  This visual evidence reinforces that curiosity-guided cooperation yields consistently cleaner tail predictions than re-weighting or neural-collapse baselines.

\paragraph{Extreme MLC (XMC)}

\begin{table}[t]
\centering
\scriptsize
\setlength{\tabcolsep}{2pt}
\begin{tabular}{lccccccccc}
\toprule
\multirow{2}{*}{Method} &
\multicolumn{3}{c}{Eurlex-4K} &
\multicolumn{3}{c}{Wiki10-31K} &
\multicolumn{3}{c}{AmazonCat-13K}\\
\cmidrule(lr){2-4}\cmidrule(lr){5-7}\cmidrule(lr){8-10}
 & P@1 & P@3 & P@5 & P@1 & P@3 & P@5 & P@1 & P@3 & P@5\\
\midrule
BCE                           & 85.60 & 72.30 & 59.10 & 86.35 & 76.05 & 66.45 & 95.80 & 82.00 & 66.30\\
PE   & 86.55 & 73.08 & 59.94 & 87.10 & 77.25 & 67.50 & 95.92 & 82.11 & 66.50\\
SRB & 86.77 & 73.32 & 60.15 & 87.35 & 77.60 & 67.82 & 95.94 & 82.18 & 66.58\\
XR-Transformer & 87.22 & 74.39 & 61.69 & 88.00 & 78.70 & 69.10 & 96.25 & 82.72 & 67.01\\
MatchXML  & \underline{88.12} & \underline{75.00} & \underline{62.22} & \underline{89.30} & \underline{80.45} & \underline{70.89} & \underline{96.50} & \underline{83.25} & \underline{67.69}\\
EUT& 87.90 & 74.86 & 61.95 & 88.75 & 79.76 & 69.95 & 96.40 & 83.05 & 67.35\\
\textbf{CD--GTMLL}                  & \textbf{89.02} & \textbf{76.14} & \textbf{63.42} & \textbf{90.28} & \textbf{82.05} & \textbf{72.10} & \textbf{96.72} & \textbf{83.45} & \textbf{68.01}\\
\bottomrule
\end{tabular}
\caption{\textbf{Precision@k (\%) on three XMC benchmarks}.  Best/second best in \textbf{bold}/\underline{underline}.}
\label{tab:xmc}
\end{table}

As shown in Table \ref{tab:xmc}, CD-GTMLL outperforms strong XMC baselines-including state-of-the-art transformer architectures and the ETU optimiser-by up to \(+1.0\%\) P@1 on Eurlex-4K and \(+1.6\%\) P@3 on Wiki10-31K, demonstrating that curiosity-guided cooperation scales seamlessly to millions of sparse labels.

\begin{figure}[t]
    \centering
    \includegraphics[width=1\linewidth]{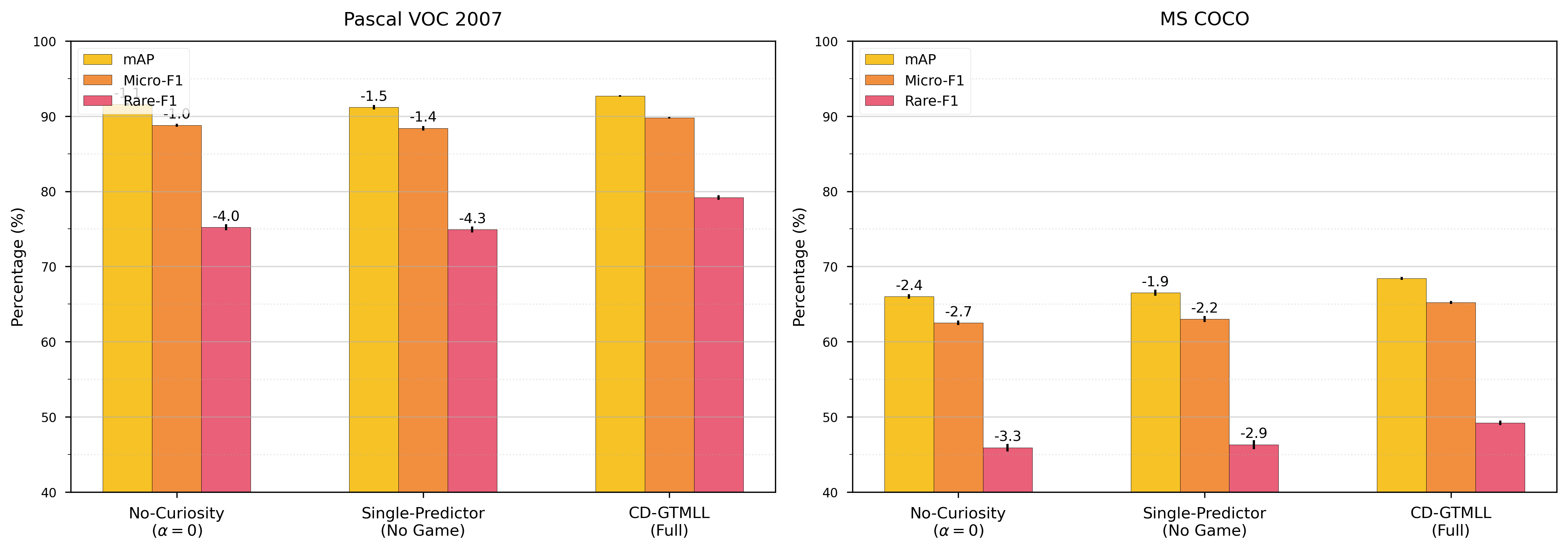}
    \caption{Ablation on Pascal VOC-2007 (left) and MS-COCO-2014 (right).
    Bars show mean\,$\pm$\,std over three seeds;
    numbers above non-full variants indicate the absolute gap to the full model.}
    \label{fig:ablation}
\end{figure}

\subsection{Ablation Studies}
\label{subsec:exp_ablation}
\paragraph{Contributions of each key component}

\begin{figure}[th] 
    \centering
    \includegraphics[width=.92\linewidth]{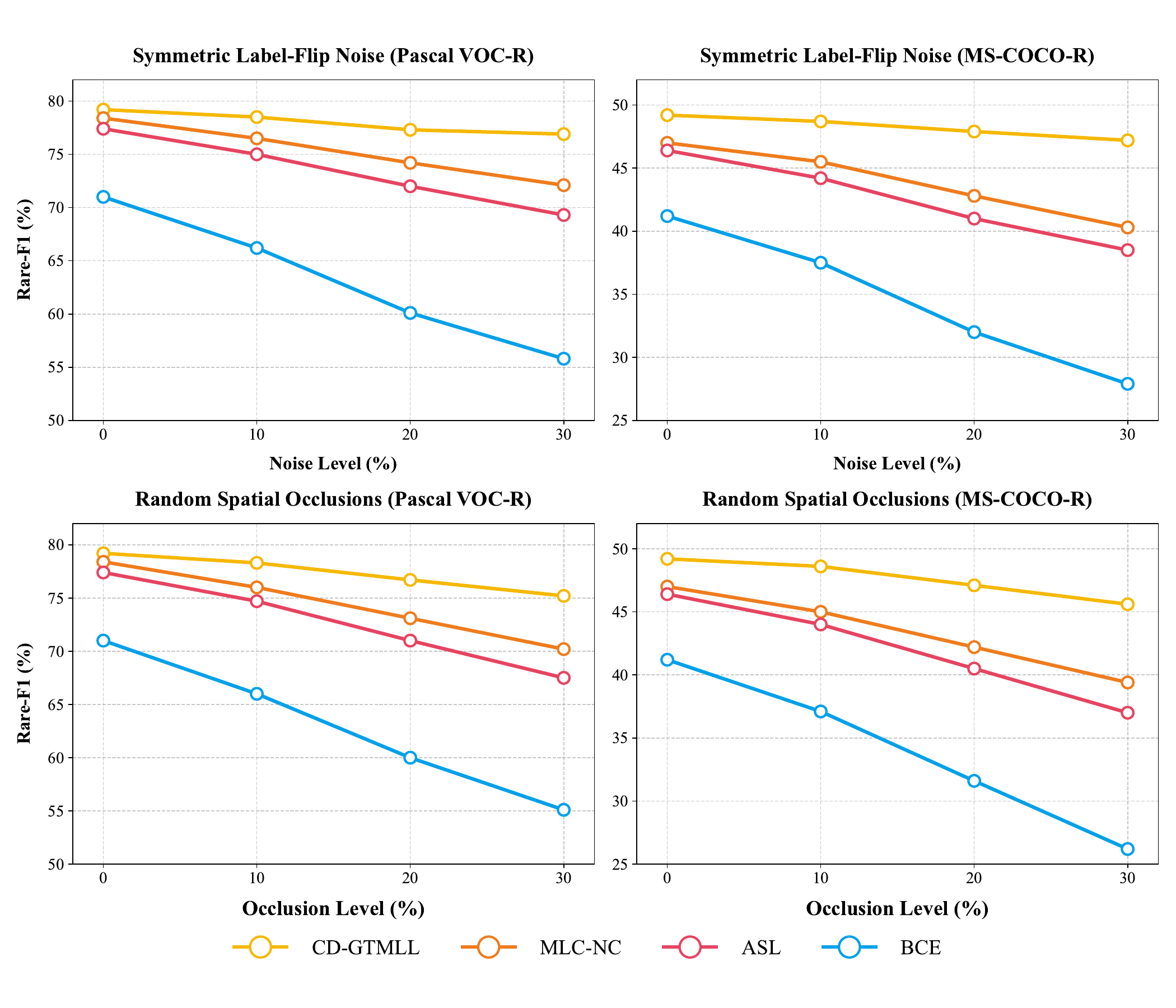} % 0.52 * 0.8 ≈ 0.416（保持原图比例）
    \caption{Absolute Rare-F1 degradation under label noise and occlusion.}
    \label{fig:noise_RO}
\end{figure}

\begin{figure}[ht] 
    \centering
    \includegraphics[width=0.6\linewidth]{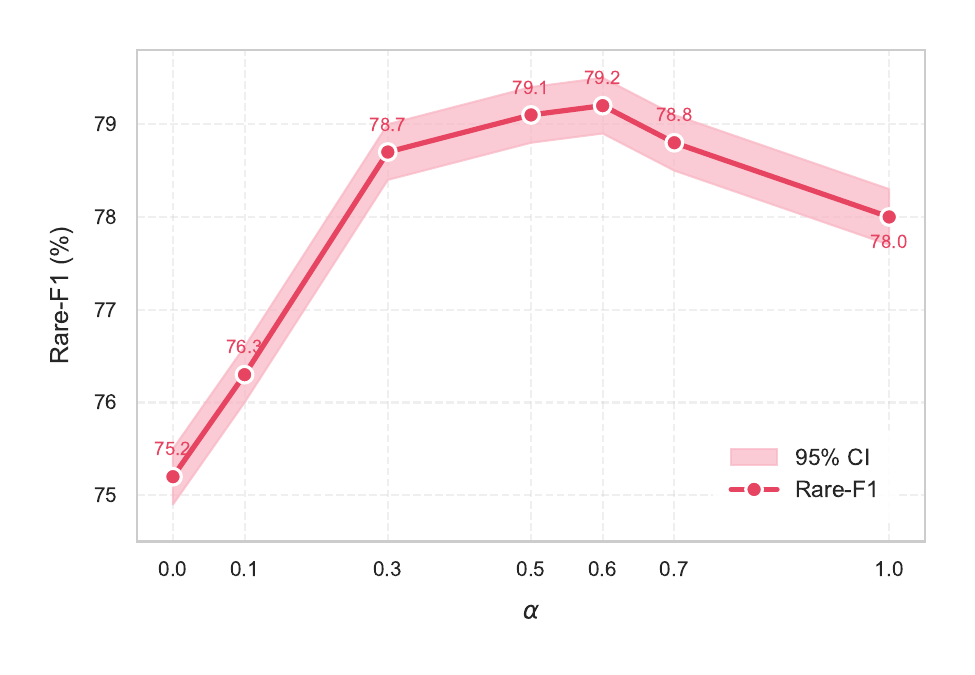} % 0.52 * 0.8 ≈ 0.416（保持原图比例）
    \vspace{-2mm}
    \caption{Rare-F1 (\%) on Pascal VOC for different \(\alpha\).}
    \label{fig:alpha_sensitivity}
\end{figure}

Figure~\ref{fig:ablation} contrasts three variants: (i) \emph{No-Curiosity} (\(\alpha\!=\!0\)), (ii) \emph{Single-Predictor} (one global head, no game), and (iii) the full CD-GTMLL.
Two consistent patterns emerge.
First, removing curiosity (orange bars) leaves head metrics almost intact (mAP drops by only \(0.3\%\) on VOC and \(2.4\%\) on COCO) but slashes Rare-F1 by \(\mathbf{4.0\%}\) and \(3.3\%\) respectively, confirming that the rarity-weighted bonus is the main driver of tail recall.
Second, collapsing the multi-player game into a single predictor (red bars) additionally harms global accuracy: mAP and Micro-F1 fall a further \(0.4\text{-}0.5\%\) on both datasets while Rare-F1 stays stuck near the No-Curiosity level.  This shows that label-space decomposition and disagreement signals are crucial for propagating the extra gradient injected by curiosity.
%The full CD-GTMLL (yellow) therefore combines both gains-recovering \(\,+1.5\%\) mAP, \(+1.4\%\) Micro-F1 and \(+4.3\%\) Rare-F1 over the best ablated variant on VOC, and larger margins on COCO-demonstrating that curiosity and cooperative game dynamics are complementary rather than interchangeable.

\paragraph{Robustness to Noisy Labels and Occlusions}
\label{subsec:robustness}
We investigate whether CD--GTMLL's curiosity signal still yields tail robustness
when training data contain \emph{(i) symmetric label-flip noise} and \emph{(ii) spatial occlusion perturbations}.
For each image we independently flip the ground-truth of every positive label with
probability $\rho\!\in\!\{0.1,0.2,0.3\}$ (\textbf{Noise-$\rho$}), or paste a random
$k\!\times\!k$ black patch covering $\gamma\!\in\!\{10\%,20\%,30\%\}$ of the pixels
(\textbf{Occ-$\gamma$}).
All models are trained from scratch on the corrupted set and evaluated on the \emph{clean} test split; we report \textit{Rare-F1}.

As shown in Figure~\ref{fig:noise_RO}, across both corruption types, \textbf{CD--GTMLL} consistently shows the \emph{smallest} Rare-F1 drop
(\(<\!2.3\%\) at 30\% noise and \(<\!4.0\%\) at 30\% occlusion),
whereas the strongest baseline (MLC-NC) loses up to \(6.3\%\)/\(8.2\%\).
This confirms that (i) the rarity-weighted curiosity continues to
push gradients toward tail labels even when annotations are partly wrong,
and (ii) disagreement-based exploration mitigates over-reliance on any single noisy cue,
making the game-theoretic ensemble inherently more noise-tolerant than static
re-weighting or single-head architectures.

\begin{figure}[th] 
    \centering
    \includegraphics[width=0.6\linewidth]{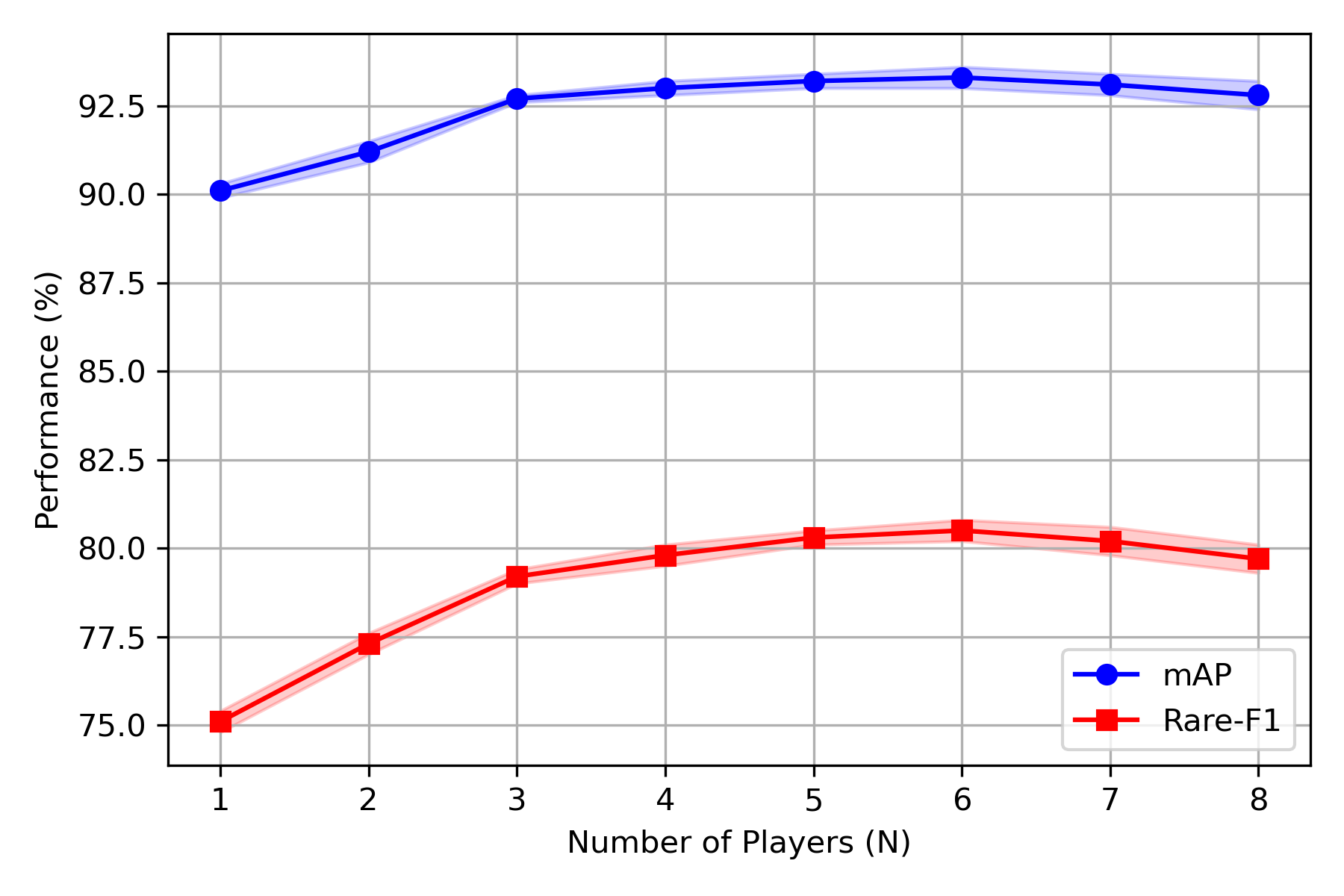} 
    \caption{Impact of $N$ (players) on Pascal VOC.}
    \label{fig:N_sensitivity}
\end{figure}

\paragraph{Sensitivity to \(\alpha\)}
Figure~\ref{fig:alpha_sensitivity} further examines the impact of varying \(\alpha\) on the Rare-F1 metric (Pascal VOC). The results are ''Mean$\pm$std'' over 3 runs.
We observe a steady increase in Rare-F1 for moderate \(\alpha\) values (e.g., 0.3--0.6), after which performance plateaus or slightly declines, possibly due to overemphasizing curiosity at the expense of overall accuracy. A similar trend emerges on MS-COCO, corroborating that tuning \(\alpha\) is important for balancing exploration and base multi-label accuracy.

\paragraph{Number of Players $N$}
As shown in Fig.~\ref{fig:N_sensitivity}, we also vary $N\in\{1,2,\dots,8\}$ on Pascal VOC and observe the following pattern:
1) Moving from $N=1$ (single) to $N=3$ significantly boosts mAP (from $90.1\%$ to $92.7\%$) and Rare-F1 (from $75.1\%$ to $79.2\%$).
2) Increasing $N$ up to around 5 or 6 continues to slightly improve Rare-F1 (peaking near $80.5\%$ at $N=6$).
3) Beyond $N=6$, adding more players actually \emph{reduces} overall performance slightly, likely due to redundant partitioning and overfitting within small subgroups of labels.

\begin{figure}[ht]
    \centering
    \includegraphics[width=0.9\linewidth]{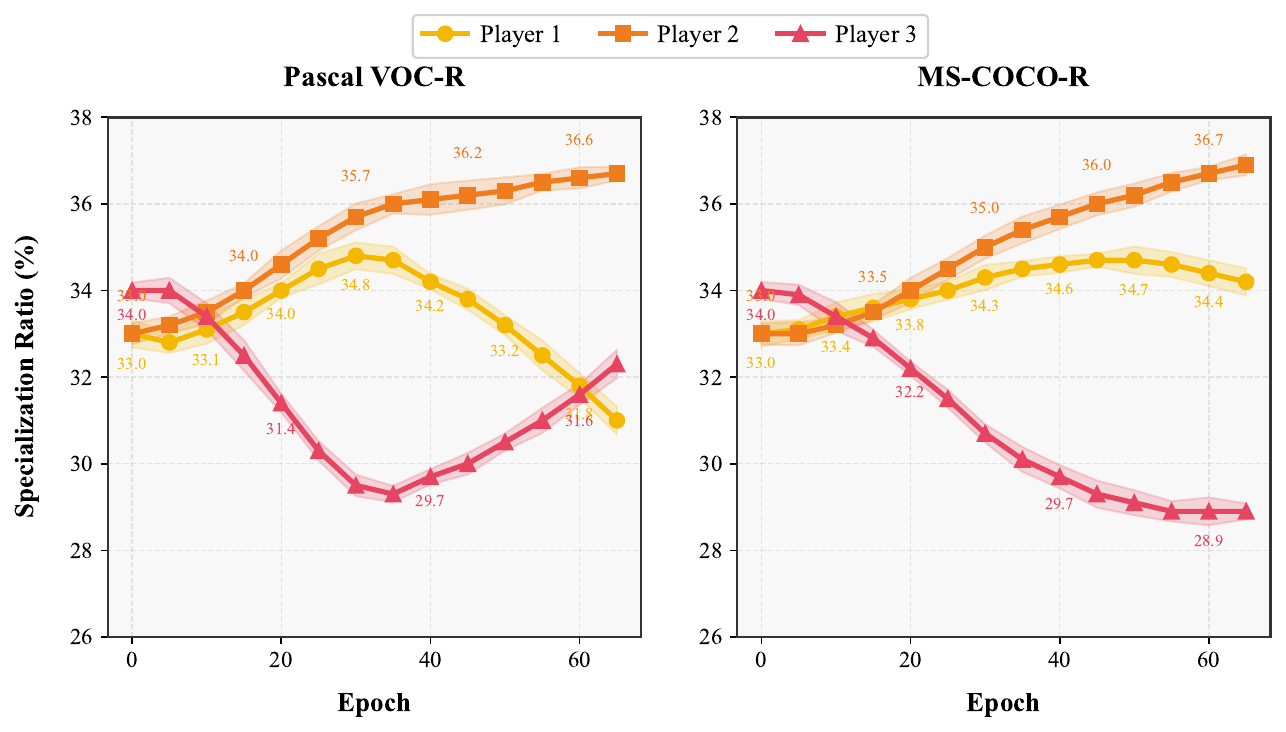} % 0.5*0.9=0.45
    \vspace{-2mm} 
    \caption{\textbf{Player specialization ratios (\%) on Pascal VOC across training epochs}. Each cell shows the percentage of labels for which that player outperforms the others.}
    \label{fig:spec_ratio}
\end{figure}

\subsection{Analysis of Multi-Player Behavior}
\label{subsec:exp_game_insight}
This section provides deeper insights into \emph{how} CD-GTMLL's cooperative multi-player structure, guided by curiosity, benefits the final performance. 

\paragraph{Player Specialization.}
Figure~\ref{fig:spec_ratio} compares \textit{Pascal VOC-R} and \textit{MS-COCO-R}.
Epochs~$0$--$15$ show balanced coverage ($\approx\!33\%$ per player).
From Epoch~20 onward, a stable \emph{division of labor} emerges:
\textbf{Player~2} rises to $\sim36\%$ by Epoch~60, while \textbf{Player~3} drops below $30\%$, making Player~2 a rare-label specialist and Player~3 a frequent-label one.
Tight confidence bands after Epoch~30 confirm convergence, and the parallel pattern across both datasets validates that curiosity-driven cooperation systematically yields complementary specialization.

\paragraph{Rare vs.\ Frequent Label Specialists}
We quantify specialization along two axes: \emph{(i) mean rank} (\#1 = best per-label accuracy) and \emph{(ii) specialization share} (\% of labels on which a player attains the best accuracy). 
Fig.~\ref{fig:specialists_rank} visualizes the mean ranks as heatmaps and Fig.~\ref{fig:specialists_share} shows the specialization shares. 
A clear pattern emerges across both datasets: \textbf{Player~2} consistently specializes on \textbf{rare} labels, while \textbf{Player~3} specializes on \textbf{frequent} labels; \textbf{Player~1} acts as a generalist. 
This division of labor mirrors the curiosity signal: the player receiving the strongest rarity-weighted gains on tail labels becomes the `tail expert,' whereas another player consolidates head performance-yielding wider tail coverage without sacrificing head accuracy.

\begin{figure}[ht]
    \centering
    \includegraphics[width=0.86\linewidth]{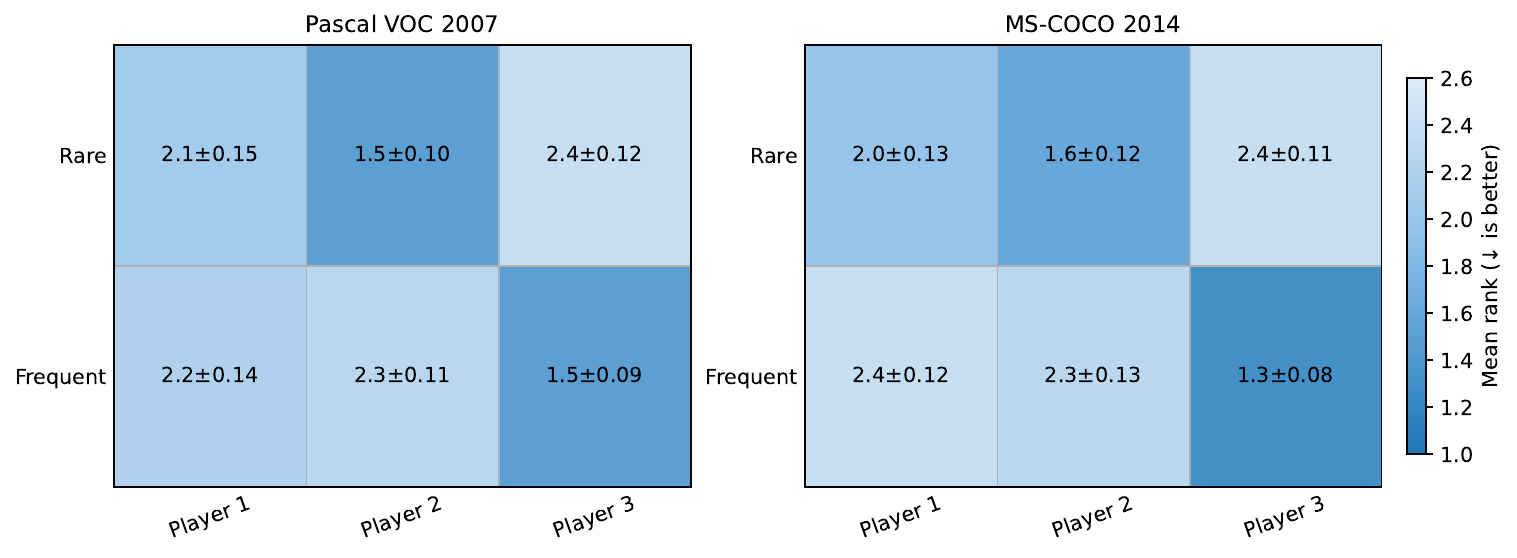}
    \caption{\textbf{Mean rank heatmaps} (lower is better). Each cell shows mean rank over labels.}
    \label{fig:specialists_rank}
\end{figure}

\begin{figure}[ht]
    \centering
    \includegraphics[width=0.86\linewidth]{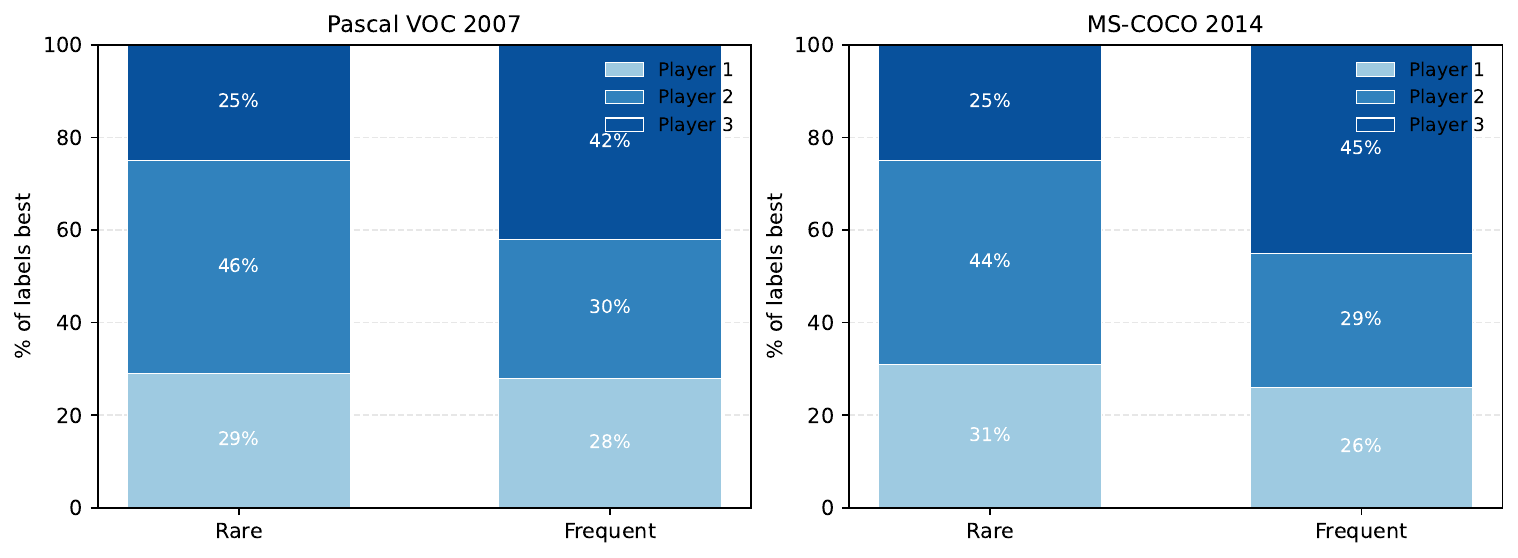}
    \caption{\textbf{Specialization shares} (\% of labels on which a player attains best accuracy). Player~2 dominates \emph{rare}, Player~3 dominates \emph{frequent} across both datasets.}
    \label{fig:specialists_share}
\end{figure}

\section{Conclusion}
\label{sec:conclusion}

We framed long-tail multi-label classification as a cooperative potential game and introduced \emph{CD-GTMLL}, where curiosity-weighted rewards drive multiple players to specialise on under-represented labels.  Theoretical analysis guarantees convergence to tail-aware equilibria and establishes an explicit link between the rarity-weighted objective and the evaluation metric \textit{Rare-F1}.  Empirically, CD-GTMLL delivers consistent state-of-the-art gains on four conventional benchmarks and three extreme-scale XMC datasets, while diagnostic studies show emergent division of labour and monotone ascent of the potential, validating the game-theoretic design.

Our implementation still relies on a fixed overlap schedule and an extra forward pass per player. Future work will explore adaptive overlap and weight sharing to further reduce overhead.

\bibliographystyle{elsarticle-harv} 
\bibliography{refs}

\newpage

\appendix

\section{Proof of Theorem \ref{thm:existence_tail}} \label{AP:A}
\setcounter{equation}{0}
\renewcommand{\theequation}{A.\arabic{equation}}

We restate the setting for completeness. The cooperative payoff is
\begin{equation}
\label{eqA:payoff}
\begin{split}
R(\boldsymbol\theta)
&=\mathbb E\!\Big[\mathcal M\big(\hat{\mathbf y}(\boldsymbol\theta),\mathbf y\big)\Big], \\
\mathcal M(\hat{\mathbf y},\mathbf y)&=\frac{1}{Z}\sum_{\ell=1}^L w_\ell\!\left[ y_\ell\log\hat p_\ell+(1-y_\ell)\log(1-\hat p_\ell)\right],
\end{split}
\end{equation}
where $w_\ell>0$, $Z=\sum_\ell w_\ell$, and $\hat p_\ell$ are fused probabilities.

\paragraph{A.1 Existence (Weierstrass)}
By Assumption~\ref{assump:continuity}(i) each $\Theta_i$ is compact and nonempty, hence $\Theta=\prod_i\Theta_i$ is compact.
The map $\boldsymbol\theta\mapsto \hat{\mathbf p}(\boldsymbol\theta)$ is continuous (Assumption~\ref{assump:continuity}(ii)); so is $\mathcal M$.
Therefore $R$ in~\eqref{eqA:payoff} is continuous; by the extreme-value theorem it attains a global maximizer on~$\Theta$.

\paragraph{A.2 Global maximizer $\Rightarrow$ Nash equilibrium}
Let $\boldsymbol\theta^\star$ be a global maximizer. For any player $i$ and any unilateral deviation $\theta_i'$,
\begin{equation}
\label{eqA:ne}
R(\theta_1^\star,\ldots,\theta_i',\ldots,\theta_N^\star)\ \le\ R(\theta_1^\star,\ldots,\theta_i^\star,\ldots,\theta_N^\star),
\end{equation}
since all players share the identical payoff $R$. Hence no player can improve by deviating: $\boldsymbol\theta^\star$ is a pure NE.

\paragraph{A.3 Tail-awareness via a strictly positive directional derivative}
Fix a tail label $\ell\in\mathcal L_T$ with $\Pr(y_\ell=1)>0$. Suppose, for contradiction, that there exists $\tau\in(\varepsilon,1-\varepsilon)$ with
\begin{equation}
\label{eqA:blind}
\hat p_\ell^\star(\mathbf x)\ \le\ \tau\qquad\text{for a.e. }\mathbf x\ \text{with }y_\ell=1.
\end{equation}
Let $S\subset\{\mathbf x:y_\ell=1\}$ be any measurable subset with $\Pr(S)>0$.
By Assumption~\ref{assump:improvability}(ii) there exist a player $k$ with $\ell\in\mathcal L_k$ and a feasible direction $\mathbf v_k$ such that the one-sided Gâteaux derivative
\begin{equation}
\label{eqA:gateaux}
\dot p_\ell(\mathbf x)\;=\;\left.\frac{\mathrm d}{\mathrm dt}\right|_{t=0^+}\hat p_\ell\!\big(\mathbf x;\theta_1^\star,\ldots,\theta_k^\star+t\mathbf v_k,\ldots,\theta_N^\star\big)
\;\ge\; c_0\,\mathbf 1_S(\mathbf x),\qquad c_0>0,
\end{equation}
and $\dot p_j(\mathbf x)\equiv 0$ for $j\neq \ell$.

Define $R(t)\triangleq R(\theta_1^\star,\ldots,\theta_k^\star+t\mathbf v_k,\ldots,\theta_N^\star)$.
Using the bounded-prediction clause (Assumption~\ref{assump:improvability}(i)) and dominated convergence, we may differentiate under the expectation:
\begin{align}
\label{eqA:Rprime}
R'(0^+)
&= \mathbb E\!\left[\frac{1}{Z}\sum_{j=1}^L w_j\Big(\frac{y_j}{\hat p_j^\star}-\frac{1-y_j}{1-\hat p_j^\star}\Big)\dot p_j\right]
= \mathbb E\!\left[\frac{w_\ell}{Z}\Big(\frac{y_\ell}{\hat p_\ell^\star}-\frac{1-y_\ell}{1-\hat p_\ell^\star}\Big)\dot p_\ell\right],
\end{align}
because $\dot p_j\equiv 0$ for $j\neq\ell$.
On $S$ we have $y_\ell=1$ and, by \eqref{eqA:blind}, $\hat p_\ell^\star\le\tau$; hence
\begin{equation}
\label{eqA:ineq}
\Big(\frac{y_\ell}{\hat p_\ell^\star}-\frac{1-y_\ell}{1-\hat p_\ell^\star}\Big)\mathbf 1_S
=\frac{1}{\hat p_\ell^\star}\,\mathbf 1_S
\ \ge\ \frac{1}{\tau}\,\mathbf 1_S.
\end{equation}
Combining \eqref{eqA:gateaux}-\eqref{eqA:ineq} yields the chain
\begin{align}
\label{eqA:chain}
R'(0^+)
&\ge \mathbb E\!\left[\frac{w_\ell}{Z}\cdot \frac{1}{\tau}\cdot c_0\,\mathbf 1_S\right]
= \frac{w_\ell c_0}{Z\tau}\,\Pr(S)
\ >\ 0.
\end{align}
Therefore there exists $t_0>0$ such that $R(t)>R(0)$ for all $t\in(0,t_0)$, contradicting the global maximality of $\boldsymbol\theta^\star$.
The contradiction proves that \eqref{eqA:blind} is false for every $\tau\in(\varepsilon,1-\varepsilon)$, i.e.,
$\Pr\big(y_\ell=1 \land \hat p_\ell^\star>\tau\big)>0$.
\hfill $\square$

\paragraph{A.4 Remarks.}
(i) If one prefers to avoid Assumption~\ref{assump:improvability}(i), an equivalent argument applies to \emph{clipped} probabilities 
$\tilde p_\ell=\min\{1-\varepsilon,\max\{\varepsilon,\hat p_\ell\}\}$; the optimization order induced by $R$ is unchanged on interior points.
(ii) The derivative formula in \eqref{eqA:Rprime} immediately extends to any surrogate $\mathcal M$ whose partial derivative
$\partial \mathcal M/\partial \hat p_\ell$ is strictly positive on tail positives $y_\ell=1$; in such cases the same inequality chain as in \eqref{eqA:chain} establishes tail-awareness.

\section{Proof of Proposition \ref{prop:rare_priority}} \label{AP.B}
\setcounter{equation}{0}
\renewcommand{\theequation}{B.\arabic{equation}}

We restate the objects used in the proof. The curiosity of player $i$ is
\begin{equation}
\label{eqB:C}
C_i(\mathbf x)=
\sum_{\ell\in\mathcal L_i}\frac{1}{1+\mathrm{freq}(\ell)}
\big[y_\ell\log p_{i\ell}(\mathbf x)+(1-y_\ell)\log(1-p_{i\ell}(\mathbf x))\big]
+\beta\,D\!\big(\pi_i(\mathbf x),\,\overline{\pi}_{-i}(\mathbf x)^{\text{stop}}\big).
\end{equation}
The per-player objective is
\begin{equation}
\label{eqB:J}
J_i(\theta_i)=R(\{\theta_j\})+\alpha\,\mathbb E\big[C_i(\mathbf x)\big].
\end{equation}
Let $z_{i\ell}$ be the logit of player $i$ for label $\ell$ so that $p_{i\ell}=\sigma(z_{i\ell})\in(0,1)$.  
We only use the \emph{rarity term} of $C_i$ for a \emph{lower bound}; the $D$-term is ignored (it can be added with $\beta\ge 0$ without invalidating positivity).

\paragraph{Step 1: derivative of the rarity term}
For any $(p,y)\in(0,1)\times\{0,1\}$,
\begin{equation}
\label{eqB:der_basic}
\frac{\partial}{\partial z}\Big(y\log p+(1-y)\log(1-p)\Big)
= \frac{\partial p}{\partial z}\Big(\frac{y}{p}-\frac{1-y}{1-p}\Big)
= y-p,
\end{equation}
since $\frac{\partial p}{\partial z}=p(1-p)$ and $\frac{p(1-p)}{p}=1-p$, $\frac{p(1-p)}{1-p}=p$.
Therefore, for $\ell\in\mathcal L_k$,
\begin{equation}
\label{eqB:der_label}
\frac{\partial}{\partial z_{k\ell}}
\mathbb E\!\left[\frac{1}{1+\mathrm{freq}(\ell)}\big(y_\ell\log p_{k\ell}+(1-y_\ell)\log(1-p_{k\ell})\big)\right]
=
\frac{1}{1+\mathrm{freq}(\ell)}\,\mathbb E\big[y_\ell-p_{k\ell}\big].
\end{equation}

\paragraph{Step 2: restriction to tail false negatives}
Fix a tail label $\ell\in\mathcal L_T$ and a threshold $\tau\in(\varepsilon,1-\varepsilon)$, and define
\begin{equation}
\label{eqB:setS}
S \;=\; \big\{\mathbf x:\ y_\ell=1,\ \hat p_\ell(\mathbf x)\le \tau\big\}.
\end{equation}
Assumption~\ref{assump:improvability}(i) (bounded predictions) implies $p_{k\ell}(\mathbf x)\le 1-\varepsilon$ almost surely.  
Hence on $S$,
\begin{equation}
\label{eqB:pos_margin}
\big(y_\ell-p_{k\ell}\big)\mathbf 1_S \;\ge\; \varepsilon\,\mathbf 1_S.
\end{equation}
Taking expectation in~\eqref{eqB:der_label} and using~\eqref{eqB:pos_margin},
\begin{equation}
\label{eqB:rarity_lb}
\frac{\partial}{\partial z_{k\ell}}\,
\alpha\,\mathbb E\!\Big[\frac{1}{1+\mathrm{freq}(\ell)}\big(y_\ell\log p_{k\ell}+(1-y_\ell)\log(1-p_{k\ell})\big)\Big]
\ \ge\ 
\alpha\,\frac{1}{1+\mathrm{freq}(\ell)}\,\varepsilon\,\Pr(S).
\end{equation}

\paragraph{Step 3: assembling the bound for $\partial J_k/\partial z_{k\ell}$}
From~\eqref{eqB:J} and the previous step we obtain
\begin{equation}
\label{eqB:final}
\frac{\partial J_k}{\partial z_{k\ell}}
=
\frac{\partial R}{\partial z_{k\ell}}
+\alpha\,\frac{\partial }{\partial z_{k\ell}}\mathbb E[C_k(\mathbf x)]
\ \ge\ 
\alpha\,\frac{1}{1+\mathrm{freq}(\ell)}\,\varepsilon\,\Pr(S),
\end{equation}
where we dropped $\tfrac{\partial R}{\partial z_{k\ell}}$ and the $D$-term for a lower bound (both can only increase the derivative when moving in the positive-$z_{k\ell}$ direction on $S$).\footnote{For the rarity-weighted logistic $R$, the chain rule yields an additional nonnegative term on $S$: 
$\tfrac{\partial R}{\partial z_{k\ell}}
=\mathbb E\!\big[\tfrac{w_\ell}{Z}\,\tfrac{1}{\hat p_\ell}\,\omega_{k\ell}\,p_{k\ell}(1-p_{k\ell})\,\mathbf 1_S\big]
\ge 0$.}
Since $\Pr(S)>0$ by assumption, the right-hand side of~\eqref{eqB:final} is strictly positive.  
Therefore $\frac{\partial J_k}{\partial z_{k\ell}}>0$, proving the claim that the curiosity term prioritizes tail labels and prevents stagnation while tail false negatives remain. \hfill$\square$

\section{Proof of Theorem \ref{thm:conv}} \label{AP.C}
\setcounter{equation}{0}
\renewcommand{\theequation}{C.\arabic{equation}}

We analyze the cyclic block gradient-ascent version of Alg.~\ref{alg:cdgtmll} on the potential
\(
\Phi(\{\theta_i\}) = R(\{\theta_i\}) + \alpha\sum_{i=1}^N \mathbb E[C_i]
\)
(cf.\ \eqref{eq:potential}). By construction $C_i$ depends only on $\theta_i$ (stop-gradient peer average), hence $\nabla_{\theta_i}\Phi \equiv \nabla_{\theta_i}J_i$.

\paragraph{Setup}
Let $\boldsymbol\theta^{(t)}=(\theta_1^{(t)},\ldots,\theta_N^{(t)})$ denote the parameter after $t$ block updates. At update $t$, we pick block $i_t\in\{1,\ldots,N\}$ in cyclic order and perform
\begin{equation}
\label{eqC:update}
\theta_{i_t}^{(t+1)} \;=\; \theta_{i_t}^{(t)} + \eta_{i_t}\,\nabla_{\theta_{i_t}}\Phi\big(\boldsymbol\theta^{(t)}\big), 
\qquad
\theta_{j}^{(t+1)} \;=\; \theta_{j}^{(t)} \ \ (j\neq i_t),
\end{equation}
with stepsize $0<\eta_{i_t}\le 1/L_{i_t}$, where $L_{i}$ is the block Lipschitz constant in Assumption~\ref{assump:reg}(ii).

\paragraph{C.1 One-step ascent guarantee.}
By block-$L_{i_t}$ smoothness (Descent/Ascent Lemma) applied to the function $\phi_i(\theta_i)\equiv \Phi(\theta_1^{(t)},\ldots,\theta_i,\ldots,\theta_N^{(t)})$,
for any $\Delta\in T_{\theta_{i_t}^{(t)}}\Theta_{i_t}$ we have
\begin{equation}
\label{eqC:lemma}
\Phi\big(\boldsymbol\theta^{(t)}+\mathbf e_{i_t}\Delta\big)\ 
\ge\ 
\Phi\big(\boldsymbol\theta^{(t)}\big)\ +\ \langle \nabla_{\theta_{i_t}}\Phi(\boldsymbol\theta^{(t)}),\,\Delta\rangle\ -\ \frac{L_{i_t}}{2}\,\|\Delta\|^2.
\end{equation}
Choosing $\Delta=\eta_{i_t}\nabla_{\theta_{i_t}}\Phi(\boldsymbol\theta^{(t)})$ and using \eqref{eqC:update} yields
\begin{equation}
\label{eqC:increase}
\Phi\big(\boldsymbol\theta^{(t+1)}\big)\ -\ \Phi\big(\boldsymbol\theta^{(t)}\big)
\ \ge\ 
\eta_{i_t}\Big(1-\frac{L_{i_t}\eta_{i_t}}{2}\Big)\ \big\|\nabla_{\theta_{i_t}}\Phi(\boldsymbol\theta^{(t)})\big\|^2.
\end{equation}
Under $0<\eta_{i_t}\le 1/L_{i_t}$, the coefficient is at least $\eta_{i_t}/2>0$. Define
\begin{equation}
\label{eqC:gamma}
\gamma_t \ \triangleq\ \eta_{i_t}\Big(1-\frac{L_{i_t}\eta_{i_t}}{2}\Big) \ \ge\ \frac{\eta_{i_t}}{2}\ > 0.
\end{equation}

\paragraph{C.2 Monotonicity and summability.}
Summing \eqref{eqC:increase} from $t=0$ to $T-1$,
\begin{equation}
\label{eqC:telescope}
\Phi\big(\boldsymbol\theta^{(T)}\big)\ -\ \Phi\big(\boldsymbol\theta^{(0)}\big)
\ \ge\ 
\sum_{t=0}^{T-1}\gamma_t\,\big\|\nabla_{\theta_{i_t}}\Phi(\boldsymbol\theta^{(t)})\big\|^2.
\end{equation}
Assumption~\ref{assump:reg}(iv) ensures $\Phi$ is bounded above by some $U<\infty$. Therefore the LHS is upper-bounded by $U-\Phi(\boldsymbol\theta^{(0)})$, implying
\begin{equation}
\label{eqC:summable}
\sum_{t=0}^{\infty}\gamma_t\,\big\|\nabla_{\theta_{i_t}}\Phi(\boldsymbol\theta^{(t)})\big\|^2\ <\ \infty.
\end{equation}
In particular, $\Phi(\boldsymbol\theta^{(t)})$ is monotone non-decreasing and convergent, and the weighted block-gradient squares are summable.

\paragraph{C.3 Vanishing block gradients along subsequences.}
Because each block $i$ is visited infinitely often in the cyclic schedule, there exists an infinite index set $\mathcal T_i\subset\mathbb N$ with $i_t=i$. From \eqref{eqC:summable} and $\gamma_t\ge \eta_{i}/2>0$ (fixed per block), we have
\(
\sum_{t\in\mathcal T_i}\big\|\nabla_{\theta_{i}}\Phi(\boldsymbol\theta^{(t)})\big\|^2<\infty,
\)
which implies
\begin{equation}
\label{eqC:vanish}
\lim_{t\in\mathcal T_i,\ t\to\infty}\big\|\nabla_{\theta_{i}}\Phi(\boldsymbol\theta^{(t)})\big\|\ =\ 0
\quad\text{for all }i=1,\ldots,N.
\end{equation}

\paragraph{C.4 Stationarity of limit points.}
Let $\boldsymbol\theta^\infty$ be any limit point of $\{\boldsymbol\theta^{(t)}\}$; there exists a subsequence $t_k\to\infty$ with $\boldsymbol\theta^{(t_k)}\to \boldsymbol\theta^\infty$. By the cyclic schedule, for each $i$ we can choose a further subsequence $t_k\in\mathcal T_i$. By continuity of $\nabla\Phi$ (Assumption~\ref{assump:reg}(i)) and \eqref{eqC:vanish},
\begin{equation}
\nabla_{\theta_i}\Phi(\boldsymbol\theta^\infty)\ =\ \lim_{t_k\to\infty}\nabla_{\theta_i}\Phi(\boldsymbol\theta^{(t_k)})\ =\ \mathbf 0,\qquad \forall i.
\end{equation}
Thus every limit point is first-order stationary.

\paragraph{C.5 extra blocks.}
If the backbone and fusion parameters $(\theta_0,\omega)$ are updated by a gradient-ascent step on $\Phi$ (Alg.~\ref{alg:cdgtmll}, last line), they form additional blocks with their own Lipschitz constants $L_0,L_\omega$ and stepsizes $0<\eta_0\le 1/L_0$, $0<\eta_\omega\le 1/L_\omega$. The argument above extends verbatim with $N$ replaced by $N+2$.

\hfill$\square$

%===================== Appendix D =====================

\section{Proof of Theorem \ref{thm:microF1_bound}}
\setcounter{equation}{0}
\renewcommand{\theequation}{D.\arabic{equation}}

We restate the cooperative payoff for the rarity-weighted logistic utility (cf.\ \eqref{eq:global_payoff}):
\begin{equation}
\label{eqD:R}
R(\{\theta_i\})=\frac{1}{Z}\sum_{\ell=1}^{L}w_\ell\Big(\Pr(y_\ell\!=\!1)\,\mathbb E[\log \hat p_\ell \mid y_\ell\!=\!1]+(1-\Pr(y_\ell\!=\!1))\,\mathbb E[\log(1-\hat p_\ell)\mid y_\ell\!=\!0]\Big),
\end{equation}
with $w_\ell>0$, $Z=\sum_{\ell}w_\ell$, and clipped probabilities $\hat p_\ell\in[\varepsilon,1-\varepsilon]$.

\paragraph{D.1 Threshold control via logistic tails.}
For any $\tau\in(\varepsilon,1-\varepsilon)$ and any $\ell$:
\begin{align}
\label{eqD:thr0}
\mathbf 1\{\hat p_\ell\ge \tau,\,y_\ell=0\}
&\le \frac{-\log(1-\hat p_\ell)}{-\log(1-\tau)}\,\mathbf 1\{y_\ell=0\},
\\
\label{eqD:thr1}
\mathbf 1\{\hat p_\ell< \tau,\,y_\ell=1\}
&\le \frac{-\log \hat p_\ell}{-\log \tau}\,\mathbf 1\{y_\ell=1\}.
\end{align}
Indeed, if $y_\ell=0$ and $\hat p_\ell\!\ge\!\tau$ then $-\log(1-\hat p_\ell)\!\ge\!-\log(1-\tau)$, hence \eqref{eqD:thr0}; \eqref{eqD:thr1} is analogous.

\paragraph{D.2 From label-wise to micro FP/FN.}
Define the tail conditional losses
\begin{equation}
\label{eqD:condloss}
L_\ell^+=\mathbb E[-\log \hat p_\ell \mid y_\ell=1],\qquad L_\ell^-=\mathbb E[-\log(1-\hat p_\ell) \mid y_\ell=0].
\end{equation}
Let $\pi_\ell=\Pr(y_\ell=1)$ and $\mu_{\mathrm{Pos},\mathrm T}=\sum_{\ell\in\mathcal L_{\mathrm T}}\pi_\ell$.  
By taking expectations in \eqref{eqD:thr0}--\eqref{eqD:thr1}, multiplying by the corresponding priors, and summing over $\ell\in\mathcal L_{\mathrm T}$,
\begin{align}
\label{eqD:microfp}
\mu_{\mathrm{FP},\mathrm T}(\tau)&=\sum_{\ell\in\mathcal L_{\mathrm T}}\Pr(\hat p_\ell\ge\tau,\,y_\ell=0)\ \le\ \frac{1}{-\log(1-\tau)}\sum_{\ell\in\mathcal L_{\mathrm T}} (1-\pi_\ell)\,L_\ell^-,
\\
\label{eqD:microfn}
\mu_{\mathrm{FN},\mathrm T}(\tau)&=\sum_{\ell\in\mathcal L_{\mathrm T}}\Pr(\hat p_\ell<\tau,\,y_\ell=1)\ \le\ \frac{1}{-\log\tau}\sum_{\ell\in\mathcal L_{\mathrm T}} \pi_\ell\,L_\ell^{+}.
\end{align}

\paragraph{D.3 Controlling tail conditional loss by the objective.}
Let
\begin{equation}
\label{eqD:St}
S_{\mathrm T}\;\triangleq\;\sum_{\ell\in\mathcal L_{\mathrm T}}\!\big(\pi_\ell L_\ell^{+}+(1-\pi_\ell)L_\ell^{-}\big),
\qquad
w_{\min,\mathrm T}\;=\;\min_{\ell\in\mathcal L_{\mathrm T}} w_\ell.
\end{equation}
From \eqref{eqD:R}, the tail contribution satisfies
\begin{equation}
\label{eqD:tailcontrib}
\sum_{\ell\in\mathcal L_{\mathrm T}} w_\ell\Big(\pi_\ell\,\mathbb E[\log \hat p_\ell | y_\ell=1]+(1-\pi_\ell)\,\mathbb E[\log(1-\hat p_\ell) | y_\ell=0]\Big)\ \ge\ Z\,R,
\end{equation}
because the (non-positive) head contribution can only increase the LHS when dropped. Multiplying by $-1$ and using $w_\ell\ge w_{\min,\mathrm T}$ on $\mathcal L_{\mathrm T}$,
\begin{equation}
\label{eqD:St_bound}
w_{\min,\mathrm T}\,S_{\mathrm T}\ \le\ -Z\,R
\quad\Longrightarrow\quad
S_{\mathrm T}\ \le\ \frac{Z}{w_{\min,\mathrm T}}\;(-R).
\end{equation}

\paragraph{D.4 Proof of Theorem \ref{thm:microF1_bound}.}
Using $\mu_{\mathrm{TP},\mathrm T}=\mu_{\mathrm{Pos},\mathrm T}-\mu_{\mathrm{FN},\mathrm T}$, we rewrite
\begin{equation}
\label{eqD:F1rewrite}
\begin{split}
\widetilde F_{\mathrm T}(\tau)
&= \frac{2(\mu_{\mathrm{Pos},\mathrm T}-\mu_{\mathrm{FN},\mathrm T})}{2(\mu_{\mathrm{Pos},\mathrm T}-\mu_{\mathrm{FN},\mathrm T})+\mu_{\mathrm{FP},\mathrm T}+\mu_{\mathrm{FN},\mathrm T}} \\
&= 1 - \frac{\mu_{\mathrm{FP},\mathrm T}+\mu_{\mathrm{FN},\mathrm T}}{2\mu_{\mathrm{Pos},\mathrm T}+\mu_{\mathrm{FP},\mathrm T}-\mu_{\mathrm{FN},\mathrm T}} \\
&\ge 1-\frac{\mu_{\mathrm{FP},\mathrm T}+\mu_{\mathrm{FN},\mathrm T}}{2\mu_{\mathrm{Pos},\mathrm T}}.
\end{split}
\end{equation}
By \eqref{eqD:microfp}--\eqref{eqD:microfn} and the elementary inequality $aA+bB\le \max\{a,b\}(A+B)$ for $a,b,A,B\ge 0$,
\begin{equation}
\label{eqD:sumfpfn}
\mu_{\mathrm{FP},\mathrm T}(\tau)+\mu_{\mathrm{FN},\mathrm T}(\tau)
\ \le\ 
\kappa(\tau)\,\sum_{\ell\in\mathcal L_{\mathrm T}}\big((1-\pi_\ell)L_\ell^{-}+\pi_\ell L_\ell^{+}\big)
\ =\
\kappa(\tau)\,S_{\mathrm T},
\end{equation}
where $\kappa(\tau)=\max\{1/[-\log(1-\tau)],1/[-\log\tau]\}$.  
Combining \eqref{eqD:F1rewrite}, \eqref{eqD:sumfpfn} and \eqref{eqD:St_bound} gives
\begin{equation}
\label{eqD:final}
\widetilde F_{\mathrm T}(\tau)\ \ge\ 1-\frac{\kappa(\tau)}{2\mu_{\mathrm{Pos},\mathrm T}}\,S_{\mathrm T}
\ \ge\ 
1-\frac{\kappa(\tau)\,Z}{2\,\mu_{\mathrm{Pos},\mathrm T}\,w_{\min,\mathrm T}}\,\bigl(-R(\{\theta_i\})\bigr),
\end{equation}
which is exactly \eqref{eq:microF1_main_bound}.

\paragraph{D.5 Tail-only payoff variant.}
If we replace $R$ by the tail-only payoff $R_{\mathrm T}$ with $Z_{\mathrm T}=\sum_{\ell\in\mathcal L_{\mathrm T}}w_\ell$, then \eqref{eqD:tailcontrib} becomes
$Z_{\mathrm T}R_{\mathrm T}\le \sum_{\ell\in\mathcal L_{\mathrm T}} w_\ell(\cdots)$ and \eqref{eqD:St_bound} refines to
$S_{\mathrm T}\le (Z_{\mathrm T}/w_{\min,\mathrm T})(-R_{\mathrm T})$, which tightens the constant in \eqref{eqD:final}.

\paragraph{D.6 Instance-level Rare--F1}
For $F_{\mathrm T}=\mathbb E\!\big[2\,\mathrm{TP}_{\mathrm T}/(2\,\mathrm{TP}_{\mathrm T}+\mathrm{FP}_{\mathrm T}+\mathrm{FN}_{\mathrm T})\big]$,  
the map $a\mapsto 2a/(2a+b+c)$ is concave in $a\ge 0$ for fixed $b,c\ge 0$.  
Conditioning on $(\mathrm{FP}_{\mathrm T},\mathrm{FN}_{\mathrm T})$ and applying Jensen gives
$F_{\mathrm T}\ge \frac{2\,\mathbb E[\mathrm{TP}_{\mathrm T}]}{2\,\mathbb E[\mathrm{TP}_{\mathrm T}]+\mathbb E[\mathrm{FP}_{\mathrm T}]+\mathbb E[\mathrm{FN}_{\mathrm T}]}
=\widetilde F_{\mathrm T}(\tau)$, so \eqref{eqD:final} also lower-bounds the instance-level Rare--F1 (possibly conservatively).
\hfill$\square$

%======================== Appendix E ========================
\section{Implementation Protocols}
\setcounter{equation}{0}
\renewcommand{\theequation}{E.\arabic{equation}}

\subsection{Implementation details of \textsc{CD--GTMLL}}
\textbf{Backbone and heads.}
All players share a single backbone encoder $\theta_0$ that produces features $\mathbf h(\mathbf x)$.
Each player $i$ owns a lightweight head $\theta_i$ that maps $\mathbf h(\mathbf x)$ to posteriors
$p_i(\mathbf x)=\pi_i(\mathbf x;\theta_i)\in[0,1]^{|\mathcal L_i|}$ (sigmoid layer).
Unless otherwise noted, heads are linear classifiers on $\mathbf h(\mathbf x)$ with bias.

\textbf{Fusion.}
We use label-wise weighted averaging as in \eqref{eq:fusion}. The weights $\omega_{i,\ell}$ are learnable \emph{per label, per player} scalars constrained by a softmax over the active players $\{i:\ell\in\mathcal L_i\}$ so that $\sum_{i:\ell\in\mathcal L_i}\omega_{i,\ell}=1$.
Weights are updated \emph{only} in the optional backbone/fusion step of Alg.~\ref{alg:cdgtmll} (outer loop), keeping the block-gradient identity $\nabla_{\theta_i}\Phi=\nabla_{\theta_i}J_i$ intact.

\textbf{Curiosity.}
We instantiate $C_i$ by the rarity-weighted log-likelihood plus a divergence $D$ between the current player and a \emph{stop-gradient} peer average, as in \eqref{eq:curiosity_smooth}.
We take $D$ as Jensen--Shannon divergence over the \emph{overlap set} $\mathcal O_i=\{\ell\in\mathcal L_i:\exists j\ne i,\ \ell\in\mathcal L_j\}$; non-overlapping coordinates are ignored in $D$.
The peer average $\overline p_{-i}$ is computed over $\{p_j\}_{j\ne i}$ from the same batch, then detached (no gradient) and optionally smoothed by an EMA with decay $0.99$ for stability in practice.

\textbf{Label-space partition and overlap.}
Let $N$ be the number of players and $\rho\in[0,1)$ the overlap ratio. We construct $\{\mathcal L_i\}_{i=1}^N$ on the training set as follows:
(i) sort labels by $\mathrm{freq}(\ell)$ (ascending); (ii) round-robin assign labels to players to balance head/tail counts per player; (iii) add overlap by assigning each tail label to an additional player chosen uniformly at random subject to per-player load balance (target coverage factor $m\approx 1+\rho$, default $\rho=0.2$); (iv) fix the partition for the entire training.

\textbf{Optimization.}
We use Adam with decoupled weight decay. Unless stated, the backbone and heads use separate learning rates $(\eta_0,\eta_{\mathrm{head}})$ with cosine decay; the fusion parameters $\omega$ share $\eta_0$.
Curiosity uses a linear warmup on the disagreement coefficient $\beta$ from $0$ to $\beta_{\max}$ in the first $10\%$ epochs; $\alpha$ is fixed.
Gradient clipping (global norm $5$) is enabled.
Mini-batch size and epochs are given in E.2.

\textbf{Thresholds and calibration.}
Training is purely probabilistic (no hard thresholding).
At evaluation, we report both a global threshold $\tau_\ell\equiv 0.5$ and label-adaptive $\tau_\ell$ tuned on the validation split (fixed at test).

\textbf{Inference.}
One forward through the backbone and all heads; fuse by \eqref{eq:fusion}; threshold to get $\hat{\mathbf y}$.
No test-time calibration or ensembling is used unless stated.

\medskip
\noindent\textbf{Default hyperparameters.}
Unless specified in E.2: $N=3$, overlap $\rho=0.2$, $\alpha=0.4$, $\beta_{\max}=0.3$, EMA decay $0.99$, weight decay $1\mathrm{e}{-4}$, gradient clip $5$.

\subsection{Dataset-specific settings}
\textbf{Conventional MLC (VOC, COCO, Yeast, Mediamill).}
Images (VOC/COCO): backbone ResNet-50 (ImageNet init), feature dim $d'=2048$; heads are linear.
$\eta_0=1\mathrm{e}{-4}$, $\eta_{\mathrm{head}}=3\mathrm{e}{-4}$, batch size $B=64$ (VOC) / $128$ (COCO), epochs $60$ / $40$ with cosine decay; $N=3$.
Tabular/text (Yeast/Mediamill): backbone MLP (2 layers, ReLU, $d'\!=\!512$) with $\eta_0=2\mathrm{e}{-3}$; $B=256$, epochs $100$; $N=3$.
Curiosity: $\alpha=0.4$ (VOC/Yeast), $0.5$ (COCO/Mediamill); $\beta_{\max}=0.3$; JSD on $\mathcal O_i$; $\rho=0.2$.

\textbf{Extreme MLC (Eurlex-4K, Wiki10-31K, AmazonCat-13K).}
Text encoder: TF-IDF ($d'\!=\!200\mathrm{k}$ sparse) + linear heads or a compact Transformer encoder (when permitted) with $d'\!=\!768$.
$N=4$, $\rho=0.15$, $\alpha=0.3$, $\beta_{\max}=0.2$, $B=512$ (sparse) or $128$ (Transformer).
Learning rates: $\eta_0=5\mathrm{e}{-4}$ (sparse) / $2\mathrm{e}{-4}$ (Transformer), $\eta_{\mathrm{head}}=5\mathrm{e}{-4}$, epochs $20$--$30$.
For $\mathrm{P}@k$, label-adaptive thresholds are not used.

\textbf{Ablations.}
No-Curiosity sets $\alpha=0$; Single-Predictor uses $N=1$; `No-Overlap' sets $\rho=0$; disagreement off sets $\beta_{\max}=0$.

\subsection{Construction of rare-focused MLC splits (``R'' variants)}
We programmatically create \emph{training-only} rare-focused variants (VOC-R, COCO-R, Yeast-R, Mediamill-R) by down-sampling \emph{positives} of the least-frequent labels while \emph{preserving} co-occurrences of other labels in the same instance. Validation and test splits remain unchanged.

\paragraph{Inputs}
Training set $\{(\mathbf x_m,\mathbf y_m)\}_{m=1}^{M}$ with multi-labels $\mathbf y_m\in\{0,1\}^{L}$; severity $\sigma\in(0,1)$; tail fraction $q$ (we use $q=0.2$).

\paragraph{Steps}
\begin{enumerate}
\item \textbf{Tail set:} compute label prevalences $\mathrm{freq}(\ell)=\frac{1}{M}\sum_m y_{m\ell}$ \emph{on the original training set}; sort ascending and take the bottom $\lfloor qL\rfloor$ labels as $\mathcal L_{\mathrm T}$ (ties resolved by fixed-seed randomness).
\item \textbf{Positive pools:} for each $\ell\in\mathcal L_{\mathrm T}$, collect its positive index set $\mathcal P_\ell=\{m:\,y_{m\ell}=1\}$ and size $P_\ell$.
\item \textbf{Down-sampling (label-wise, instance-preserving):} for each $\ell\in\mathcal L_{\mathrm T}$, sample a subset $\mathcal S_\ell\subset\mathcal P_\ell$ of size $\lfloor \sigma P_\ell\rfloor$ uniformly without replacement (fixed seed), and \emph{flip only that label} to negative: for all $m\in\mathcal S_\ell$, set $y_{m\ell}\leftarrow 0$ while keeping all other labels $\{y_{mj}:j\ne \ell\}$ unchanged.
\item \textbf{Bookkeeping:} record the achieved down-sampling ratio $\widehat\sigma_\ell=|\mathcal S_\ell|/P_\ell$ and report the new training prevalences. No instance is removed; we only edit selected label entries.
\end{enumerate}

\paragraph{Severity levels}
We use dataset-specific severities to match the main tables:
VOC-R: $\sigma\in\{0.30,0.50\}$; COCO-R: $\sigma\in\{0.30,0.40\}$; Yeast-R: $\sigma\in\{0.40,0.50\}$; Mediamill-R: $\sigma\in\{0.40,0.50\}$.
Results are averaged over 3 seeds for the sampling step.

\subsection{Metrics: formal definitions}
\paragraph{Setup}
Let the test set be $\{(\mathbf x_m,\mathbf y_m)\}_{m=1}^{M}$ with $\mathbf y_m\in\{0,1\}^{L}$, and let $\hat{\mathbf p}_m\in[0,1]^L$ be predicted label probabilities.
Write $y_{m\ell}\in\{0,1\}$, $\hat p_{m\ell}\in[0,1]$, and $\hat y_{m\ell}(\tau_\ell)=\mathbf 1\{\hat p_{m\ell}\ge \tau_\ell\}$ for a (global or per-label) threshold $\tau_\ell\in(0,1)$.
For a label $\ell$, denote the set of its positives by $\mathcal P_\ell=\{m:\,y_{m\ell}=1\}$ and its count by $P_\ell=|\mathcal P_\ell|$.

\paragraph{Micro- and Macro-F1}
Define global (micro) counts
\begin{equation}
\label{eqE:micro_counts}
\mathrm{TP}_{\text{micro}}=\sum_{m,\ell} y_{m\ell}\hat y_{m\ell},\quad
\mathrm{FP}_{\text{micro}}=\sum_{m,\ell} (1-y_{m\ell})\hat y_{m\ell},\quad
\mathrm{FN}_{\text{micro}}=\sum_{m,\ell} y_{m\ell}(1-\hat y_{m\ell}).
\end{equation}
Then
\begin{equation}
\label{eqE:micro_f1}
\begin{aligned}
\mathrm{Prec}_{\text{micro}} &= \frac{\mathrm{TP}_{\text{micro}}}{\mathrm{TP}_{\text{micro}}+\mathrm{FP}_{\text{micro}}}, \\
\mathrm{Rec}_{\text{micro}} &= \frac{\mathrm{TP}_{\text{micro}}}{\mathrm{TP}_{\text{micro}}+\mathrm{FN}_{\text{micro}}}, \\
\mathrm{F1}_{\text{micro}} &= \frac{2\,\mathrm{Prec}_{\text{micro}}\mathrm{Rec}_{\text{micro}}}{\mathrm{Prec}_{\text{micro}}+\mathrm{Rec}_{\text{micro}}}.
\end{aligned}
\end{equation}
Per-label counts are
\begin{equation}
\label{eqE:perlabel_counts}
\mathrm{TP}_\ell=\sum_m y_{m\ell}\hat y_{m\ell},\quad
\mathrm{FP}_\ell=\sum_m (1-y_{m\ell})\hat y_{m\ell},\quad
\mathrm{FN}_\ell=\sum_m y_{m\ell}(1-\hat y_{m\ell}),
\end{equation}
\begin{equation}
\label{eqE:perlabel_f1}
\mathrm{Prec}_\ell=\frac{\mathrm{TP}_\ell}{\mathrm{TP}_\ell+\mathrm{FP}_\ell},\quad
\mathrm{Rec}_\ell=\frac{\mathrm{TP}_\ell}{\mathrm{TP}_\ell+\mathrm{FN}_\ell},\quad
\mathrm{F1}_\ell=\frac{2\,\mathrm{Prec}_\ell\mathrm{Rec}_\ell}{\mathrm{Prec}_\ell+\mathrm{Rec}_\ell},
\end{equation}
and the macro-F1 is
\begin{equation}
\label{eqE:macro_f1}
\mathrm{F1}_{\text{macro}}=\frac{1}{L}\sum_{\ell=1}^{L}\mathrm{F1}_\ell.
\end{equation}

\paragraph{Rare-F1 (tail macro-F1)}
Let label prevalences $\mathrm{freq}(\ell)=\frac{1}{M}\sum_{m} y_{m\ell}$ be computed on the \emph{training} set.  
Let $\mathcal L_{\mathrm T}\subset\{1,\ldots,L\}$ be the tail set formed by the bottom $20\%$ labels in $\mathrm{freq}(\ell)$ (ties resolved as described in E.2).  
Then
\begin{equation}
\label{eqE:rare_macro}
\mathrm{Rare\text{-}F1}\;\triangleq\;\frac{1}{|\mathcal L_{\mathrm T}|}\sum_{\ell\in\mathcal L_{\mathrm T}}\mathrm{F1}_\ell.
\end{equation}
(Unless otherwise stated, Rare-F1 denotes the \emph{macro} average on the tail subset.)

\paragraph{Mean Average Precision (mAP)}
For each label $\ell$, rank the $M$ instances by $\hat p_{m\ell}$ in descending order, obtaining a permutation $m_{(1)},\ldots,m_{(M)}$.  
Let $P_\ell=|\mathcal P_\ell|$. The average precision (all-points interpolation) is
\begin{equation}
\label{eqE:ap_label}
\mathrm{AP}_\ell=\frac{1}{P_\ell}\sum_{t=1}^{M} \mathrm{Prec}_\ell(t)\cdot \mathbf 1\{y_{m_{(t)}\ell}=1\},
\quad
\mathrm{Prec}_\ell(t)=\frac{1}{t}\sum_{u=1}^{t}\mathbf 1\{y_{m_{(u)}\ell}=1\}.
\end{equation}
Labels with $P_\ell=0$ on the test split are ignored in the mean. The mAP is
\begin{equation}
\label{eqE:map}
\mathrm{mAP}=\frac{1}{|\{\ell:P_\ell>0\}|}\sum_{\ell:P_\ell>0} \mathrm{AP}_\ell.
\end{equation}

\paragraph{Extreme MLC: $\mathrm{P}@k$}
For instance $m$, let $\hat S_m^{(k)}$ be the top-$k$ labels by $\hat p_{m\ell}$.  
Then
\begin{equation}
\label{eqE:pak}
\mathrm{P}@k=\frac{1}{M}\sum_{m=1}^{M}\frac{|\hat S_m^{(k)}\cap \{\,\ell:\,y_{m\ell}=1\,\}|}{k},
\qquad k\in\{1,3,5\}.
\end{equation}

\end{document}